\documentclass[twoside]{article}

\usepackage[accepted]{aistats2026}
\usepackage[round]{natbib}

\usepackage{hyperref}
\usepackage{url}
\usepackage{amsmath}
\usepackage{tikz}    
\usepackage{quantikz}
\usepackage{amsfonts}
\usepackage{amssymb}
\usepackage{algpseudocode}
\usepackage{comment}
\usepackage{pifont}
\usepackage{subcaption}
\usepackage{amsthm}
\usepackage{booktabs}
\usepackage{multirow}
\usepackage{algorithm}
\usepackage{booktabs}
\usepackage{subcaption}
\usepackage{tabularx}

\newcommand{\params}{\boldsymbol{\theta}}
\newtheorem{theorem}{Theorem}
\newtheorem{lemma}{Lemma}

\begin{document}

% If your paper is accepted and the title of your paper is very long,
% the style will print as headings an error message. Use the following
% command to supply a shorter title of your paper so that it can be
% used as headings.
%
%\runningtitle{I use this title instead because the last one was very long}
\runningtitle{WSBD: Freezing-Based Optimizer for Quantum Neural Networks}
\runningauthor{Kverne, Akewar, Huo, Patel, Bhimani}

% If your paper is accepted and the number of authors is large, the
% style will print as headings an error message. Use the following
% command to supply a shorter version of the author names so that
% they can be used as headings (for example, use only the surnames)
%
%\runningauthor{Surname 1, Surname 2, Surname 3, ...., Surname n}

\twocolumn[

\aistatstitle{WSBD: Freezing-Based Optimizer for \\ Quantum Neural Networks}

% \aistatsauthor{ 
%   Christopher Kverne\textsuperscript{$\dagger$} \And 
%   Mayur Akewar\textsuperscript{$\dagger$} \And 
%   Yuqian Huo\textsuperscript{$\ddagger$} \AND 
%   Tirthak Patel\textsuperscript{$\ddagger$} \And 
%   Janki Bhimani\textsuperscript{$\dagger$} 
% }

% \aistatsauthor{
%   \begin{tabular}{c}
%     Christopher Kverne\textsuperscript{$\dagger$} \hspace{3em} Mayur Akewar\textsuperscript{$\dagger$} \hspace{3em} Yuqian Huo\textsuperscript{$\ddagger$} \\[0ex] 
%     Tirthak Patel\textsuperscript{$\ddagger$} \hspace{3em} Janki Bhimani\textsuperscript{$\dagger$}
%   \end{tabular}
% }

% \aistatsaddress{ 
%   \textsuperscript{$\dagger$}Florida International University \And 
%   \textsuperscript{$\ddagger$}Rice University 
% }

\aistatsauthor{
  Christopher Kverne \\ \textnormal{Florida International University} \And
  Mayur Akewar \\ \textnormal{Florida International University} \AND
  Yuqian Huo \\ \textnormal{Rice University} \And
  Tirthak Patel \\ \textnormal{Rice University} \And
  Janki Bhimani \\ \textnormal{Florida International University}
}

% Keep this command, but leave it empty
\aistatsaddress{}
]

\begin{abstract}
The training of Quantum Neural Networks (QNNs) is hindered by the high computational cost of gradient estimation and the barren plateau problem, where optimization landscapes become intractably flat. To address these challenges, we introduce Weighted Stochastic Block Descent (WSBD), a novel optimizer with a dynamic, parameter-wise freezing strategy. WSBD intelligently focuses computational resources by identifying and temporarily freezing less influential parameters based on a gradient-derived importance score. This approach significantly reduces the number of forward passes required per training step and helps navigate the optimization landscape more effectively. Unlike pruning or layer-wise freezing, WSBD maintains full expressive capacity while adapting throughout training. Our extensive evaluation shows that WSBD converges on average 63.9\% faster than Adam for the popular ground-state-energy problem, an advantage that grows with QNN size. We provide a formal convergence proof for WSBD and show that parameter-wise freezing outperforms traditional layer-wise approaches in QNNs. Project page: \url{https://github.com/Damrl-lab/WSBD-Stochastic-Freezing-Optimizer}. %We also provide a formal convergence proof, directly linking stochastic parameter freezing with guaranteed descent and showing that parameter-wise freezing is more effective than traditional layer-wise approaches for QNNs. %53.7\% faster on average than standard gradient-based optimizers for the MNIST and parity problem, and 62.3\% faster than Adam for the ground-state-energy problem an advantage that grows with QNN size. We provide a formal proof of convergence and show that parameter-wise freezing is more effective than traditional layer-wise approaches for QNNs.
\end{abstract}

\section{INTRODUCTION}
The promise of quantum neural networks (QNNs) \citep{abbas2021power} is currently shackled by a fundamental roadblock: the cost and impracticality of training them. This challenge is twofold. First, the primary method for gradient evaluation, the Parameter-Shift Rule (PSR) \citep{crooks2019gradients, wierichs2022general}, demands at least two full circuit evaluations for every single parameter at every training step. This creates a crippling computational bottleneck that scales linearly with the model size, making gradient-based optimization prohibitively expensive for large-scale QNNs. Adding to this issue is the notorious barren plateau problem, where optimization landscapes flatten exponentially with qubit count, causing gradients to vanish and stall training progress \citep{mcclean2018barren, larocca2025barren}. Trainers are thus caught in a double bind: the gradients that are essential for optimization are both incredibly expensive to compute and often non-impactful when obtained. To reduce the computational burden of training, parameter freezing is a well-established strategy in classical machine learning. However, classical freezing methods, which are often static or operate at a coarse, layer-wise level for transfer learning, are fundamentally unsuited for the unique optimization landscape of QNNs. The highly entangled nature of quantum circuits means a parameter's influence is not neatly confined to a layer, and its importance can change dramatically throughout training. Therefore, a new approach is required, one that is granular, dynamic, and designed specifically to accelerate QNN training from scratch. To address this, we introduce Weighted Stochastic Block Descent (WSBD), a novel optimization algorithm that employs a dynamic, parameter-wise freezing strategy. By intelligently and temporarily allocating computational resources, WSBD directly confronts the high cost of gradient estimation in a manner tailored for the quantum domain. Our major contributions are:

\textbf{A Novel QNN-Specific Freezing Optimizer:} We propose WSBD, which uses a gradient-derived importance score to temporarily freeze the least influential parameters during training. Its design overcomes the limitations of classical methods by offering a dynamic, stochastic, and parameter-wise approach that preserves the model's full expressive capacity while significantly reducing the number of required circuit evaluations. \textbf{Scalable Efficiency:} Our extensive experiments show WSBD is significantly more efficient, saving hundreds of thousands of circuit evaluations compared to optimizers like Adam. We demonstrate that this advantage scales with QNN size, highlighting WSBD's effectiveness for larger, more computationally demanding models. \textbf{Theoretical Guarantees and Critical Insights:} We provide a formal proof of convergence for WSBD, establishing its theoretical soundness. Furthermore, our ablation studies provide direct evidence for our central claim, revealing a critical insight for the field: a granular, parameter-wise freezing strategy is more effective than traditional layer-wise approaches for QNNs. %is on average 47.9\% more effective than traditional layer-wise approaches for QNNs.

% \textbf{Source code for WSBD can be found on GitHub} (\url{https://github.com/Damrl-lab/WSBD-Stochastic-Freezing-Optimizer}).

%\subsection{Strengthen scalable advantage}
%Clarify shots vs FP and advantage over ADAM. Highlight that 10 qubits is larger than thought. WSBD ADAM - ADAM, WSBD SGD SGD.
%Future work do for more problems and more different QNN designs
\section{BACKGROUND AND RELATED WORK}
\label{sec:related}
\subsection{QNNs and Key Training Obstacles}

A QNN operates on the $n$-qubit ground state, $\ket{0}^{\otimes n}$ in an exponentially large Hilbert space $\mathcal{H}$. For a given input, $\hat{x_i} \in \mathbb{R}^m$, an encoding circuit, $A(\hat{x_i})$, applies a sequence of quantum gates whose parameters are functions of $\hat{x_i}$ mapping the classical input to a quantum state. Following this, a variational circuit or ansatz, $U(\params)$, applies a sequence of gates organized into L layers. The total unitary transformation of the ansatz is the product of the operation of each layer:
\begin{equation}
U(\params)  = \prod_{l=L}^{1} U_l(\boldsymbol{\theta}_l)W_l
\end{equation}
where $U_l(\params_l)$ represents a set of parametrized gates and $W_l$ represents a set of constant  gates (e.g., CNOT, CZ, H, etc.) in a given layer $l$. The final state is given by $\ket{\psi(\params, \hat{x_i})} = U(\params)A(\hat{x_i})\ket{0}^{\otimes n}$. To extract a classical output % The goal of training is to find the optimal parameters $\params^*$ that minimize an objective function $\mathcal{C}(\params)$. This 
we must measure the system through a Hermitian observable, $M$, on the final state, which is the expectation value of the observable:
\begin{equation}
f(\params, \hat{x_i}) = \langle \psi(\params, \hat{x_i}) | M | \psi(\params, \hat{x_i}) \rangle %= \langle \bra{0}^{\otimes n} A(x)^\dagger U(\params)^\dagger M U(\params) A(x) \ket{0}^{\otimes n}
\end{equation}
Now we can define an objective or cost function $\mathcal{C}(\params, \hat{x_i})$ which aims to quantify the difference of the label $\hat{y}_{true}$ and prediction $f(\params, \hat{x_i})$. A classical optimizer then iteratively updates $\params$ to minimize the cost. However, this process faces two significant hurdles: \textbf{(1) High Cost of Gradient Estimation: } As direct differentiation is infeasible, %due to the exponentially many terms in the Hilbert space $(\mathbb{C}^2)^{\otimes n}$; rather, 
gradients are computed using the Parameter-Shift Rule (PSR) \citep{crooks2019gradients}. The gradient for a parameter $\theta_k$ is calculated as:
\begin{equation}
\label{eq:param_shift_general}
\frac{\partial \mathcal{C}(\params, \hat{x_i})}{\partial \theta_k} = \varsigma_P[\mathcal{C}(\params + \frac{\pi}{4\varsigma_P}\mathbf{\vec{e}_k}, \hat{x_i}) - \mathcal{C}(\params - \frac{\pi}{4\varsigma_P}\mathbf{\vec{e}_k}, \hat{x_i})]
\end{equation}
Here $\varsigma_P$ represents a constant (for the Pauli-matrices, $\varsigma_P$ = $\frac{1}{2}$), $\mathbf{\vec{e}_k}$ is the $k$-th standard basis vector (a vector with 1 in the $k$-th position and 0s elsewhere, meaning the shift is applied only to $\theta_k$ for each optimization step). Although exact, this means that computing $|\params|$ gradients requires $2|\params| + 1$ total evaluations, one to evaluate the cost, $\mathcal{C}(\params, \hat{x_i})$, and $2|\params|$ to evaluate all shifts ($\theta_k \pm \frac{\pi}{2} \text{ } \forall \theta_k \in \params$), making gradient-based optimization expensive for large-scale QNNs. \textbf{(2) Barren Plateaus:} For many QNN architectures, particularly deep or wide ones, the optimization landscape becomes flat \citep{mcclean2018barren}. This is formally expressed as a "barren plateau," where the variance of the cost function's gradient vanishes exponentially with the number of qubits, $n$:
\begin{equation}
\text{Var}\left[\frac{\partial \mathcal{C}(\params, \hat{x_i})}{\partial \theta_k}\right] \propto \mathcal{O}\left(\frac{1}{\text{exp}(n)}\right) \quad %\text{or} \propto \mathcal{O}\left(\frac{1}{\text{poly}(n)}\right)
\end{equation}
This means gradients become exponentially small, providing no meaningful direction for optimization and effectively stalling the training process. \textit{These combined challenges make scaling QNNs a formidable task, necessitating new strategies to make training more efficient and effective.}

\begin{comment}
\begin{figure}[h]
\centering
\vspace{-2mm}
\begin{quantikz}[row sep=0.15cm, column sep=0.3cm]
\lstick{\ket{0}} &
\gate[4, style={cyan!30, draw=black, thick}]{\parbox{1.2cm}{\centering \Large $A(\vec{x})$ \\ \small Encoding}} & 
\gate[4, style={orange!40, draw=black, thick}]{\parbox{2.5cm}{\centering\Large U($\params$) \\ \small Variational Layers}} &
\meter{} \\
\lstick{\ket{0}} & & & \meter{} \\
\setwiretype{n} \vdots & & & \vdots \\
\lstick{\ket{0}} & & & \meter{}
\end{quantikz}
\vspace{-2mm}
\caption{QNN Structure Consisting of Data Encoding, Variational Layers, and Measurement.}
\label{fig:qnn_struc}
\vspace{-5mm}
\end{figure}
\end{comment}

\subsection{Strategies for Accelerating QNN Training}

Several research directions have aimed to mitigate the high cost of QNN training through adaptive ansatz methods and circuit optimization \citep{huo2025revisiting}. One popular approach is pruning and compression strategies which seeks to simplify the QNN by permanently removing non-influential parameters. This method, similar to freezing, reduces the number of forward passes required by 2 for each parameter removed. Frameworks like \citep{qnn_compression} and algorithms such as \citep{qadaprune, kverne2025quantum, pect} have shown that removing a fraction of parameters can speed up training and improve performance. Similarly, some work has explored gradient compression techniques \citep{alistarh2017qsgd, lin2017deep}, however this fails to address the core computational bottleneck of gradient estimation itself. A more dynamic strategy is parameter freezing, which is the focus of this work. Unlike pruning, freezing does not change the circuit architecture; instead, it temporarily deactivates the training of selected parameters, concentrating computational effort on a smaller, active subset. Prior work on quantum parameter freezing is limited. Layer-wise freezing has been employed to incrementally grow QNNs \citep{layerwiselearning}, and dropout-inspired methods have been proposed to improve generalization \citep{qnn_drop}. However, these approaches do not aim to accelerate training from scratch, nor do they leverage fine-grained, adaptive control over which parameters are frozen during training. In contrast, freezing strategies are well-established in classical machine learning. They are central to transfer learning \citep{hu2022lora, yosinski2014transferable, Kornblith2018DoBI, he2019rethinking} and have also been used to reduce training time by avoiding backpropagation through certain layers \citep{brock2017freezeout}. However, the idea of selectively freezing parameters in QNNs, from initialization and throughout training, as a way to improve training efficiency remains unexplored. \textit{To address this gap, we draw inspiration from classical ML to introduce a method that accelerates QNN training without permanently altering the circuit architecture. We propose a dynamic, parameter-wise freezing strategy that intelligently adapts throughout the training process by concentrating computational resources on the most impactful parameters at each learning stage.} 
\section{WEIGHTED STOCHASTIC BLOCK DESCENT}
\label{sec:optim}

\begin{algorithm*}
\caption{Weighted Stochastic Block Descent (WSBD)}
\label{alg:wsbd}
\begin{algorithmic}[1]
\State \textbf{Input:} Initial parameters $\params^{(0)}$, Freeze Threshold $\lambda_f \in [0,1)$, Learning Rate $\eta > 0$, Training Window $\tau \in \mathbb{N}_+$
\State \textbf{Initialize:} Importance scores $\mathcal{I}_p(\theta_k) \gets 0$ for all $k=1, \dots, |\params|$, Active set $\mathbb{A} \gets \params$
\While{not converged}
    \For{$t = 1, \dots, \tau$} \Comment{Iterate within the training window}
        \State Sample data instance $(\hat{x}_t, \hat{y}_t)$
        \State Initialize gradient vector $g_t \gets \mathbf{0} \in \mathbb{R}^{|\params|}$
        \For{each parameter $\theta_k \in \mathbb{A}$} \Comment{Compute gradients only for active parameters}
            \State $g_{t,k} \gets \varsigma_P[\mathcal{C}(\params + \frac{\pi}{4\varsigma_P}\mathbf{\vec{e}_k}, \hat{x}_t) - \mathcal{C}(\params - \frac{\pi}{4\varsigma_P}\mathbf{\vec{e}_k}, \hat{x}_t)]$
        \EndFor
        \State $\params^{(t+1)} \gets \params^{(t)} - \eta \cdot g_t$ \Comment{Update parameters with any optimizer (frozen parameter grads are 0)}
        \For{each parameter $\theta_k \in \mathbb{A}$} \Comment{Update importance scores for active parameters}
            \State $\mathcal{I}_p^{(t+1)}(\theta_k) \gets \mathcal{I}_p^{(t)}(\theta_k) + g_{t,k}$ \Comment{Using Sum of Gradients metric}
        \EndFor
    \EndFor
    
    \State \Comment{After window, stochastically select the new active set}
    \State Define selection probability $P = [p_1, \cdots, p_k, \cdots, p_{|\params|}]$: $p_k \leftarrow \frac{|\mathcal{I}_p(\theta_k)| + \epsilon}{\sum_{i=1}^{|\params|} (|\mathcal{I}_p(\theta_i)| + \epsilon)} $ for all $k$ %\Comment{$\epsilon$ ensures $p_k>0$}
    \State $N_{active} \gets \lceil (1 - \lambda_f) \cdot |\params| \rceil$ \Comment{Calculate size of the new active set}
    \State $\mathbb{A} \gets \text{Sample}(N_{active}, \params, P)$ \Comment{Sample $N_{active}$ params using weights $P$}
    
    \For{each parameter $\theta_k \in \mathbb{A}$} \Comment{Reset scores for newly active parameters}
        \State $\mathcal{I}_p(\theta_k) \gets 0$
    \EndFor
\EndWhile
\State \textbf{Return:} Optimized parameters $\params$
\end{algorithmic}
\end{algorithm*}

In this section, we introduce our freezing-based QNN-specific optimizer, \textbf{Weighted Stochastic Block Descent (WSBD)}. The key idea is to intelligently focus computational resources on subsets of parameters that are most influential at different stages of the optimization process. Our optimizer is built upon the PSR (Eq. \ref{eq:param_shift_general}) and can be adapted to \textbf{work with any classical-based gradient optimizer} such as SGD, Adam, etc. At the core of WSBD is an importance score, $\mathcal{I}_p$, which quantifies the influence of each parameter $\theta_k$ over a given training window $\tau$. While several metrics for this score are possible, we selected the Sum of Gradients which accumulates the gradients of each parameter in a given window:

\begin{equation}
    \mathcal{I}_p(\theta_k) = \sum_{t=1}^\tau \frac{\partial \mathcal{C}(\params^{(t)}, \hat{x_t})}{\partial \theta_k}
\end{equation}

Our empirical analysis, detailed in Appendix \ref{app:imp_scores}, confirms this metric provides the best balance of performance and computational efficiency. It robustly identifies parameters with a consistent impact, unlike variance-based scores which can be misleading in barren plateaus, and avoids the significant overhead of second-order methods. After computing all importance scores we take its absolute value and add a small constant $\epsilon = 10^{-8}$ such that $\mathcal{I}_p(\theta_k) = |\mathcal{I}_p(\theta_k)| + \epsilon$, ensuring each importance score is greater than zero which will be crucial for proving convergence for WSBD (Appendix \ref{sec:proof}). The WSBD optimizer is summarized in Algorithm \ref{alg:wsbd}. The algorithm proceeds in windows of size $\tau$. Within each window (lines 4-14), it performs the following steps for each training instance: It starts by computing the gradients for the parameters in the active set $\mathbb{A}$ using PSR (lines 7-9). It then updates the active parameters using a classical optimizer (e.g., SGD, Adam, etc.) (line 10). Finally it accumulates the gradients for each active parameter to update their respective importance scores $\mathcal{I}_p$ (lines 11-13). After $\tau$ steps, the algorithm uses the calculated importance scores to stochastically determine which parameters to freeze. A probability distribution is formed where the probability, $p_k$, of a parameter $\theta_k$
remaining active is its normalized importance score: $p_k = \frac{|\mathcal{I}_p(\theta_k)| + \epsilon}{\sum_{i=1}^{|\params|}(|\mathcal{I}_p(\theta_i)| + \epsilon)}$ (line 16). Consequently, parameters with lower importance scores are assigned a higher probability of being frozen. Based on this distribution, the algorithm stochastically selects a new set of frozen parameters, with the total size of the set determined by the freeze percentile $\lambda_f$ (line 18). This stochastic selection, in contrast to deterministically freezing the parameters with the lowest scores, promotes greater exploration of the optimization landscape by preventing permanent freezing. 

A crucial feature of WSBD is how it manages the importance scores across freezing cycles. When a parameter becomes part of the new active set, its importance score is reset to zero (lines 19-21).  However, the scores of parameters that remain frozen are cached (i.e., preserved). This reset mechanism is critical for the algorithm's adaptiveness, as it enables rapid assessment of newly activated parameters' current contributions while preventing legacy bias from parameters that were important early in training but have since become less influential. Our ablation studies confirm this empirically in Sec.~\ref{sec:eval}, showing that the reset and stochastic freezing mechanisms are more effective. The classical overhead of our optimizer is linear, scaling as $O(|\params|)$ with the number of parameters for updating importance scores and executing the stochastic freeze. This is insignificant when compared to the substantial cost of the quantum circuit evaluations required for gradient estimation which is of order $O(\mathcal{M}2|\params|)$ where$ \mathcal{M}$ represents the number of times each forward pass is run on quantum hardware (commonly called a shot or circuit evaluation) which is needed as the state collapse will only give you a single output opposed to the entire output distribution needed to evaluate the cost. The higher value for $\mathcal{M}$ the more accurate estimate of $f(\params, \hat{x_i})$ is given (commonly $\mathcal{M} \sim 1000$). We optimized the freeze threshold $\lambda_f$ and training window $\tau$ via grid search, finding the best performance with $\lambda_f$=70\% and $\tau$=100. The full tuning process is detailed in Appendix \ref{app:imp_scores}.
\section{THEORETICAL CONVERGENCE FRAMEWORK} \label{sec:theory}

We extend the theoretical guarantees of WSBD to a broad class of optimizers $\mathcal{O}$. WSBD acts as a stochastic coordinate mask, scaling the expected descent direction without amplifying variance. 

\textbf{Generalized Update Rule.} Let $\mathcal{O}$ generate an update vector $u_t$ based on $\nabla\mathcal{C}^{(t)}$. The WSBD-augmented update is:
\begin{equation}
    \theta^{(t+1)} = \theta^{(t)} - \eta_t (\delta^{(t)} \odot u_t)
\end{equation}

\textbf{Assumptions.} We assume: 
(1) $\mathcal{C}$ is $L$-smooth (proven for QNNs in Appendix \ref{sec:proof}, Theorem \ref{thm:elem_bound}) and bounded below by $\mathcal{C}^*$; 
(2) $\mathcal{O}$ produces a valid descent direction: $\mathbb{E}[\langle \nabla\mathcal{C}^{(t)}, u_t \rangle \mid \mathcal{F}_t] \ge c_1 \|\nabla\mathcal{C}^{(t)}\|^2$; 
(3) The update magnitude is bounded: $\mathbb{E}[\|u_t\|^2 \mid \mathcal{F}_t] \le K$; and 
(4) The mask $\delta^{(t)}$ is independent of $u_t$ with minimum selection probability $p_{\min} > 0$ which is ensured by $\epsilon$.

\begin{proof}
By $L$-smoothness, the objective decreases as:
\begin{equation*}
\mathcal{C}^{(t+1)} \le \mathcal{C}^{(t)} - \eta_t \langle \nabla\mathcal{C}^{(t)}, \delta^{(t)} \odot u_t \rangle + \frac{L\eta_t^2}{2} \|\delta^{(t)} \odot u_t\|^2
\end{equation*}

Taking expectations conditioned on $\mathcal{F}_t$, we utilize the independence of the mask $\delta^{(t)}$ and the descent assumption to bound the first-order term by $- \eta_t p_{\min} c_1 \mathbb{E}[\|\nabla\mathcal{C}^{(t)}\|^2]$. The second-order term is a contraction, bounded by $\frac{L\eta_t^2}{2}K$. Summing these expectations over $T$ steps:

\begin{multline*}
p_{\min} c_1 \sum_{t=0}^{T-1} \eta_t \mathbb{E}[\|\nabla\mathcal{C}^{(t)}\|^2] \le 
\mathcal{C}^{(0)} - \mathcal{C}^* + \frac{LK}{2} \sum_{t=0}^{T-1} \eta_t^2
\end{multline*}

Given standard step-size conditions ($\sum \eta_t = \infty, \sum \eta_t^2 < \infty$), the RHS remains finite while the LHS accumulates. This implies $\liminf_{t\to\infty} \mathbb{E}[\|\nabla\mathcal{C}^{(t)}\|^2] = 0$. (See Appendix \ref{sec:proof} for full derivations including SGD and AMSGrad variants). This means WSBD will not change the convergence guarantees of other optimizers as the active parameter set changes throughout.
\end{proof}
\section{EXPERIMENTAL SETUP}
\label{sec:setup}
We benchmark the performance of our proposed WSBD optimizer across three representative tasks: MNIST classification, the parity problem, and a Variational Quantum Eigensolver (VQE) ground-state problem. Our goal is to demonstrate that WSBD integrates with classical optimizers to improve convergence efficiency. To this end, we compare it against a suite of standard optimizers (summarized in Table~\ref{tab:optimizers}) using the QNN architecture shown in Fig.~\ref{fig:pqc}. To assess both near-term practicality and future potential, we conduct all experiments in two settings: an idealized, noise-free simulation and a realistic noisy environment that emulates current hardware by incorporating decoherence, depolarizing noise, and readout errors calibrated from the IBM Heron quantum device~\citep{IBMQuantum2025}.
%In this section we describe how we evaluate the performance of our proposed WSBD optimizer. We compare it against a wide range of existing optimizers on three representative tasks, namely classification with the MNIST dataset, the parity problem, and the ground state VQE problem. The variational quantum architecture used in the experiments is shown in Fig.~\ref{fig:pqc}. A summary of the optimizers and their main characteristics is provided in Table~\ref{tab:optimizers}. The goal is to show that WSBD can be plugged into any classical optimizer and improve convergence efficiency. We evaluate WSBD under both idealized environments (noise-free simulations representing a perfect quantum computer) and noisy environments (realistic conditions emulating current quantum hardware). In the noisy setting, each gate, idle time, and measurement is subject to decoherence, depolarizing noise, and readout errors calibrated from IBM Heron device parameters \cite{IBMQuantum2025}. Comparing noisy and ideal environments provides insight into both the current practicality of WSBD on near-term hardware and its projected performance on future fault-tolerant devices.

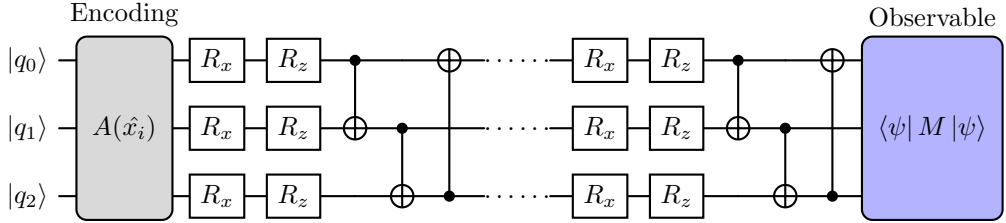
\begin{figure*}[!htbp]
\centering
\begin{quantikz}[row sep=0.3cm, column sep=0.3cm]
\lstick{$|q_0\rangle$} & \gate[3, style={rounded corners,fill=gray!30, inner xsep=2pt, label=above:Encoding}]{A(\hat{x_i})} & \gate{R_x} & \gate{R_z} & \ctrl{1} & & \targ{} &  \hdots \hdots & \gate{R_x} & \gate{R_z} & \ctrl{1} & & \targ{} & \gate[3, style={rounded corners,draw=black,fill=blue!30, inner xsep=2pt, label=above:Observable}]{\bra{\psi} M \ket{\psi}} \\
\lstick{$|q_1\rangle$} & & \gate{R_x} & \gate{R_z} & \targ{} & \ctrl{1} & &  \hdots \hdots & \gate{R_x} & \gate{R_z} & \targ{} & \ctrl{1} & & \\
\lstick{$|q_2\rangle$} & & \gate{R_x} & \gate{R_z} & & \targ{} & \ctrl{-2} &  \hdots \hdots & \gate{R_x} & \gate{R_z} & & \targ{} & \ctrl{-2}  &
\end{quantikz}
\caption{VQA architecture used in this study with data encoding, variational layers, and measurement. The circuit begins with an encoding block $A(\hat{x_i})$ that maps classical data to quantum states, followed by alternating parametrized rotations and entangling CNOT gates. The final measurement observable $\bra{\psi} M \ket{\psi}$ extracts the computational result.}
\label{fig:pqc}
\end{figure*}

\begin{table*}[!hbt]
    \centering
    \caption{Comparison of optimizer characteristics. $|\params|$ is the total number of parameters, and $\mathbb{A}$ is the set of active (non-frozen) parameters, where $|\mathbb{A}| \leq |\params|$. $\mathcal{M}$ represents the number of shots (circuit evaluations) each forward pass is run on the circuit, where a higher value for $\mathcal{M}$ gives more precise outputs.}
    \label{tab:optimizers}
    
    \begin{tabular}{l c c c c c l}
        \toprule
        & & \multicolumn{2}{c}{\textbf{Selection}} & \multicolumn{2}{c}{\textbf{Freezing}} & \textbf{Circuit Evals} \\
        \cmidrule(lr){3-4} \cmidrule(lr){5-6}
        \textbf{Optimizer} & \textbf{Gradient} & \textbf{Stochastic} & \textbf{Det.} & \textbf{Layer} & \textbf{Parameter} & \textbf{Per Step}\\
        \midrule
        SGD & \checkmark & & & & & $\mathcal{M}(2|\params|+1)$ \\
        WSBD-SGD & \checkmark & \checkmark & & & \checkmark & $\mathcal{M}(2|\mathbb{A}|+1)$ \\
        DBD-SGD & \checkmark & & \checkmark & & \checkmark & $\mathcal{M}(2|\mathbb{A}|+1)$ \\
        SBD-SGD & \checkmark & \checkmark & & & \checkmark & $\mathcal{M}(2|\mathbb{A}|+1)$ \\
        L-WSBD-SGD & \checkmark & \checkmark & & \checkmark & & $\mathcal{M}(2|\mathbb{A}|+1)$ \\
        WSBD-NO-RESET & \checkmark & \checkmark & & & \checkmark & $\mathcal{M}(2|\mathbb{A}|+1)$ \\
        \midrule
        ADAM & \checkmark & & & & & $\mathcal{M}(2|\params|+1)$ \\
        WSBD-ADAM & \checkmark & \checkmark & & & \checkmark & $\mathcal{M}(2|\mathbb{A}|+1)$ \\
        \midrule
        SPSA & & \checkmark & & & & $2\mathcal{M}$ \\
        NELDER-MEAD & & & \checkmark & & & $\text{Evals} \in (\mathcal{M}, \mathcal{M}|\params|]$ \\
        BAYESIAN & & \checkmark & & & & $\mathcal{M}$ \\
        \bottomrule
    \end{tabular}%
    
\end{table*}

\subsection{Tasks and Datasets}
We evaluate WSBD on three tasks that represent different challenges common in QML. The first task is MNIST image classification %using the digits $\{1,2,3,4\}$ 
\citep{lecun2002gradient}, where each image is preprocessed with Principal Component Analysis to match the number of qubits and is then encoded into the quantum state space using angle encoding. For a reduced feature vector $\hat{x_i}$ the operator $A(\hat{x_i})$ applies single-qubit rotations: $A(\hat{x_i})\ket{0}^{\otimes n} = \left( \bigotimes_{i=1}^{n} R_X(\hat{x_i}) \right) \ket{0}^{\otimes n}.$ The second task is the parity problem. For a binary input string $S = \langle x_1 x_2 \dots x_n \rangle$ with $x_k \in \{0, 1\}$ the parity function is $P(S) = (\sum_{k=1}^n x_k) \text{ mod } 2$. The encoding operator $A(\hat{x_i})_{Parity}$ applies a Pauli-X gate to the $k$-th qubit if $x_k = 1$. The resulting initial state is: $A(\hat{x_i})_{Parity}\ket{0}^{\otimes n} = \left( \bigotimes_{k=1}^{n} X^{x_k} \right) \ket{0}^{\otimes n}.$

%These two tasks provide complementary benchmarks: MNIST is a data-driven classification problem while the parity task is a computational function problem. Together they test the versatility of WSBD across different optimization landscapes. 

The third task evaluates WSBD in a noisy, hardware-realistic setting through a Variational Quantum Eigensolver (VQE). VQE aims to approximate the ground state energy of a given Hamiltonian $H$ by preparing a parametrized quantum state $|\psi(\params)\rangle$ and minimizing the energy expectation value $E(\params) = \bra{\psi(\params)}M\ket{\psi(\params)}$. Here we study the one-dimensional transverse-field Ising model (TFIM) Hamiltonian, defined for $n$ qubits as: $H_{TFIM} = -J \sum_{i=1}^{n-1}Z_iZ_{i+1}-h\sum_{i=1}^nX_i$. This Hamiltonian is widely used as a benchmark in VQE studies because it is non-trivial yet efficiently simulatable, making it ideal for analyzing optimizer performance under noise \citep{peruzzo2014variational, kandala2017hardware}. Unlike MNIST and parity, the VQE problem does not require an explicit data encoding stage; the ansatz circuit directly parameterizes candidate ground states. The cost function is simply the expectation value of $H_{\text{TFIM}}$, and optimization seeks to minimize this energy. Because each parameter-shift evaluation involves stochastic measurement outcomes and is further distorted by decoherence and gate errors, this task provides a stringent test of WSBD under realistic noise conditions. We simulate noise using calibration data from the IBM Heron backend, including amplitude damping ($T_1$), phase damping ($T_2$), depolarizing noise on one- and two-qubit gates, idle errors, and readout assignment errors. This models the primary error channels of near-term superconducting hardware \citep{IBMQuantum2025}.

\subsection{Experimental Design}
\label{sec:exp_design}
We conduct experiments on quantum neural networks with varying architectures, including a 4-qubit 2-layer model, an 8-qubit 3-layer model, and a 10-qubit 5-layer model. These architectures represent challenging, large-scale benchmarks by current quantum standards, chosen specifically to demonstrate our optimizer's advantage where training costs become prohibitive. For MNIST we use an ansatz with $R_X$ and $R_Z$ single-qubit rotations and cyclic entanglement, while for the parity problem the encoding $A(\hat{x_i})_{Parity}$ is used and only a single qubit is measured. All experiments are performed on simulated quantum hardware using the PennyLane framework \citep{bergholm2018pennylane}. We repeat each experiment five times for optimizers with stochastic elements (such as the freezing mechanism in WSBD). WSBD consistently shows stability, with a coefficient of variation in convergence between 0--3\%. Reported results are averaged over these runs. For the noisy ground-state energy experiments, we use the same circuit structure and test 1-, 2-, and 4-qubit systems with varying depth.  

% For all experiments we report \emph{forward passes} (FPs) to target performance---either classification accuracy (Table~\ref{tab:mnist_accuracy}) or ground-state energy (Table~\ref{tab:vqe_percent_reduction}). This choice is motivated by the structure of quantum training: each update requires a full circuit evaluation per input state, unlike classical ML where batching amortizes cost. Accuracy curves are therefore noisy at small qubit counts and finite shots, and for VQE tasks ``accuracy'' is not even defined. FP-to-target provides a \emph{fair, hardware-independent metric} of efficiency, directly tied to the dominant parameter-shift training cost, which scales as $O(\mathcal{M} \cdot 2|\params|)$ per update (with $\mathcal{M}$ shots and $|\params|$ trainable parameters). WSBD adds only $O(|\params|)$ classical overhead, negligible when $\mathcal{M}$ is in the hundreds to thousands. 
While wall-clock time is generally confounded by backend throughput, queueing delays, and simulator performance, we performed a small \emph{real-hardware calibration} to contextualize FP savings. Using 100 random QNNs (4q-2l, 8q-3l, 10q-5l), we measured forward-pass latencies on the real \textbf{IBM Heron R2 processor} as \textbf{1.33 $\pm$ 0.24 seconds of QPU execution time} and \textbf{4.74 $\pm$ 0.76 seconds of end-to-end wall time} (including submission and return). Thus, \emph{each forward pass saved by WSBD corresponds to roughly five seconds of wall time} on current hardware \citep{IBMQuantum2025}. These numbers can fluctuate depending on the backend and system load, since queue lengths and calibration states vary, but they nevertheless illustrate that FP reductions translate directly into practical time savings. For this reason, we retain FP-to-target as our \emph{primary evaluation metric}, since it provides a reproducible, hardware-agnostic measure of algorithmic efficiency, while our calibration confirms that FP reductions yield substantial wall-clock benefits in practice.

%%%% OLD SECTION ON TUNING  HERE

\subsection{Optimizers Considered}
We compare WSBD against both classical gradient-based and gradient-free optimizers. Standard baselines include stochastic gradient descent (SGD) and Adam \citep{kingma2014adam}. We also evaluate Nelder-Mead \citep{nelder1965simplex}, Simultaneous Perturbation Stochastic Approximation (SPSA) \citep{spall2002multivariate}, and Bayesian optimization \citep{jones1998efficient}, which are widely used in quantum machine learning due to their efficiency in the absence of explicit gradients.  In addition, we perform ablation studies with several WSBD variants. Deterministic Block Descent (DBD) freezes parameters with the lowest importance scores deterministically; this variant is to showcase the value of the stochastic freezing element in WSBD. Stochastic Block Descent (SBD) randomly freezes parameters without importance scores, this variant showcases that WSBD performance improvement doesn't come from freezing directly but its intelligent importance tracking. Layer-wise WSBD freezes whole layers based on summed importance scores. WSBD without reset does not reset importance scores of recently active parameters, this variant showcases that resetting scores enables more rapid assessment of parameter importance to speed up training. These comparisons help isolate the contribution of each design choice in WSBD. 

% \subsection{Hyperparameter Settings}
% \label{sec:hyerparams_chosen}
% We tune WSBD using grid search on two parameters, the freezing threshold $\lambda_f$ and the time window size $\tau$, using the MNIST task on a 10-qubit 5-layer model. First, we fix $\tau = 100$ and vary $\lambda_f$ across $\{0\%, 10\%, 30\%, 50\%, 70\%, 90\%\}$. Then we fix $\lambda_f = 70\%$ and vary $\tau$ across $\{25, 100, 300, 500\}$. The best performance is obtained with $\lambda_f = 70\%$ and $\tau = 100$. This supports the hypothesis that freezing a large fraction of less important parameters accelerates convergence and that shorter time windows help the optimizer adapt more effectively. We notice that these hyperparameter values are generally stable with good performance improvements for $\tau$ being between 100-500 and $\lambda_f$ being between 30-90\%. See Appendix for a detailed comparison of importance metrics and hyperparameters \ref{app:imp_scores}.

% \subsection{Evaluation Metrics}
% For MNIST classification we use cross entropy loss as the evaluation metric. For the parity problem we measure the accuracy of predicting the parity function. In addition, we assess convergence speed by recording the number of iterations and the training time needed to reach a target performance. These metrics allow us to fairly compare WSBD with other optimizers across tasks.
\section{EVALUATION AND RESULTS}
\begin{figure*}
    \centering
    \includegraphics[width=\linewidth]{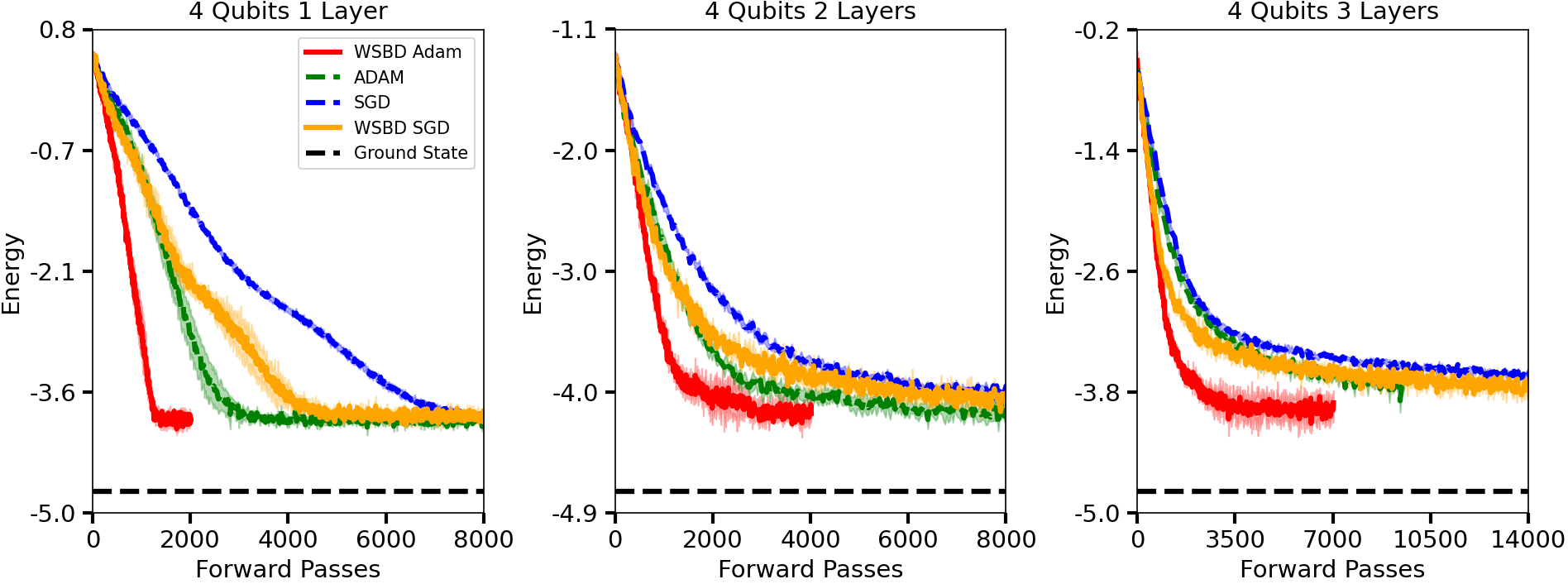}
    \caption{Training curves using Adam, SGD and their WSBD counterpart optimizers on the VQE problem. The black dotted line represents the ground state energy each optimizer aims to reach (closer is better). Each optimizer was trained until convergence. For the full training curves of all QNN sizes see Figure \ref{fig:vqe_all}. Table \ref{tab:vqe_percent_reduction} shows the percentage reduction in forward passes needed to converge using both WSBD SGD and WSBD Adam.}
    \label{fig:samp_vqe_training_curve}
\end{figure*}

\begin{table*}[t]
\centering
\caption{Side-by-side results. Left: percentage reduction in forward passes (FP) to reach target energy for the noisy VQE task (higher is better). Right: max accuracy for MNIST and Parity, plus speedup when adding WSBD. See Figures \ref{fig:vqe_all}, \ref{fig:qnn_combined} and \ref{fig:combined_accuracy_training_curves} for the respective training curves on all problems.}
\label{tab:vqe_mnist_side_by_side}

\renewcommand{\arraystretch}{1.15}
\setlength{\tabcolsep}{6pt}

\begin{subtable}[t]{0.47\textwidth}
\centering
\caption{VQE: \% FP reduction needed to reach target energy. Target energy represents the converged values of Adam and SGD.}
\label{tab:vqe_percent_reduction}
\begin{tabular}{lcc}
\toprule
\textbf{Model} & \textbf{WSBD-SGD} & \textbf{WSBD-Adam} \\
\textbf{Size}  & \textbf{v.s. SGD} & \textbf{v.s. Adam} \\
\midrule
1q, 2l & \textbf{25.9\%} & \textbf{43.9\%} \\
1q, 4l & \textbf{20.6\%} & \textbf{83.1\%} \\
1q, 8l & \textbf{21.5\%} & \textbf{69.0\%} \\
2q, 1l & \textbf{28.1\%} & \textbf{55.3\%} \\
2q, 2l & \textbf{58.2\%} & \textbf{82.8\%} \\
2q, 4l & \textbf{40.0\%} & \textbf{65.5\%} \\
4q, 1l & \textbf{34.0\%} & \textbf{59.9\%} \\
4q, 2l & \textbf{23.3\%} & \textbf{39.2\%} \\
4q, 3l & \textbf{46.4\%} & \textbf{76.5\%} \\
\bottomrule
\end{tabular}
\end{subtable}
\hfill
\begin{subtable}[t]{0.50\textwidth}
\centering
\caption{MNIST \& Parity: Max accuracy reached for each optimizer. The 'Speedup' compares the forward passes required for each optimizer to reach its own reported maximum accuracy. As the maximum accuracies achieved by WSBD and its baseline counterparts were consistently similar, this metric provides a direct comparison of training efficiency.}
\label{tab:mnist_accuracy}
% 6 columns total: l + 5 c's
\begin{tabular}{l@{\hskip 2pt}cc|@{\hskip 2pt}cc|@{\hskip 2pt}c}
\toprule
 & \multicolumn{2}{c}{Standard} & \multicolumn{2}{c}{WSBD} & Speedup \\
 
\cmidrule(lr){2-3}\cmidrule(lr){4-5}
\textbf{Model} & SGD & Adam & SGD & Adam & (SGD, Adam) \\
\midrule
\multicolumn{6}{c}{\textit{MNIST Classification Max Accuracy}} \\
4q 2l & 58.1\% & 70.2\% & 59.9\% & 70.4\% & $80.3\%, 46.2\%$ \\
8q 3l & 49.6\% & 68.5\% & 49.9\% & 68.7\% & $38.5\%, 30.2\%$ \\
10q 5l & 56.6\% & 73.7\% & 57.5\% & 74.2\% & $29.3\%, 17.6\%$ \\
\midrule
\multicolumn{6}{c}{\textit{Parity Problem Max Accuracy}} \\
4q 2l  & 50\% & 100\% & 100\% & 100\% & $\infty, 7.6\%$ \\
8q 3l & 100\% & 100\% & 100\% & 100\% & 48.8\%, 3.0\% \\
10q 5l & 50\% & 100\% & 50\% & 100\% & N/A, 20.7\% \\
\bottomrule
\end{tabular}
\end{subtable}

\end{table*}

% \begin{table}[t]
% \centering
% \caption{Percentage reduction in forward passes (FP) to reach target loss when replacing SGD/Adam with WSBD variants. Positive values indicate fewer FP (better).}
% \label{tab:wsbd_percent_reduction}
% \renewcommand{\arraystretch}{1.15}
% \setlength{\tabcolsep}{6pt}
% \begin{tabular}{llcc}
% \toprule
% \textbf{Task} & \textbf{Model} & \textbf{WSBD-SGD} & \textbf{WSBD-Adam}\\
% & \textbf{Size} & \textbf{v.s. SGD} & \textbf{v.s. Adam} \\
% \midrule
% MNIST & 4q, 2l & \textbf{63.6\%} & \textbf{29.5\%} \\
% MNIST & 8q, 3l & \textbf{60.1\%} & \textbf{61.0\%}\\
% MNIST & 10q, 5l & \textbf{60.9\%} & \textbf{45.8\%}\\
% Parity & 4q, 2l & \textbf{54.7\%} & \textcolor{red}{-6.1\%}\\
% Parity & 8q, 3l & \textbf{29.0\%} & \textbf{6.9\%} \\
% Parity & 10q, 5l & N/A & \textbf{12.5\%} \\
% \bottomrule
% \end{tabular}
% \end{table}

\label{sec:eval}
In this section we evaluate WSBD across three representative QNN tasks. Our analysis focuses on three themes: (i) scalable efficiency as circuit size increases, (ii) robustness under noise in hardware-realistic environments, and (iii) the importance of WSBD’s design choices through ablation studies. Together these experiments provide both empirical validation and insight into the mechanisms that make WSBD effective.

\subsection{Scalable Efficiency and Practical Savings in QNN Training}
A central motivation for WSBD is to accelerate training in large QNNs, where the cost of gradient evaluation scales linearly in the number of parameters $|\params|$. Table \ref{tab:vqe_to_taget_energy_DETAILED} shows the raw number of forward passes (FP) required for convergence across models of increasing size and how many forward passes are saved for all problems. While improvements are already visible on 4-qubit systems, the relative benefit of WSBD compounds with scale: on 8- and 10-qubit models, WSBD-SGD and WSBD-Adam save tens of thousands of forward passes compared to their baselines. The reason for this trend is twofold. First, larger circuits require proportionally more parameter-shift evaluations per update, which magnifies the savings achieved when WSBD focuses only on a subset of parameters. Second, barren plateaus — regions where gradients vanish exponentially with qubit count — become more prevalent as the system grows. In these landscapes, standard optimizers expend computational effort on parameters whose gradients are effectively noise. By concentrating updates on parameters with consistently strong gradient signals, WSBD both reduces overhead and steers optimization away from flat directions. This explains why WSBD not only improves convergence speed but also enables successful optimization where SGD or Adam begin to stagnate (e.g., the 4-qubit parity task in Fig. \ref{fig:xor_4q} where WSBD-SGD converged and SGD didn't). In summary, WSBD’s efficiency gains are not merely numerical savings but a reflection of how parameter-wise freezing interacts with the geometry of high-dimensional quantum optimization landscapes. The advantage therefore scales with problem size, making WSBD particularly suitable for the next generation of deeper QNNs.

Based on our real-hardware calibration on the IBM Heron R2 processor (Section~\ref{sec:exp_design}), each forward pass requires on average $\approx$5 seconds of end-to-end wall time. For example, in the VQE ground-state energy problem, WSBD-Adam reaches the target energy in roughly $3,252$ fewer forward passes on average compared to Adam, corresponding to a reduction of $\approx 4.5$ hours of wall-clock compute. For the MNIST problem the savings are more extreme with WSBD-Adam reaching its maximum accuracy while saving on average 36.3 hours of wall-clock compute and WSBD-SGD saving 89.9 hours of compute. Similarly for the parity problem, WSBD-Adam reached 100\% accuracy saving 8.8 hours of wall-clock compute, while WSBD-SGD saved \textbf{216} hours. These back-of-the-envelope calculations, while approximate and dependent on backend conditions, underscore the practical value of FP reductions: training speedups of tens of percent (Tables \ref{tab:vqe_mnist_side_by_side}) directly translate into \emph{hours of quantum compute saved}. Even at modest qubit counts, these reductions are practically significant: current quantum hardware is both limited and in high demand, with academic users often allocated only a few hours of device access per month. Saving several hours of wall-clock compute per training run can therefore determine whether an experiment is feasible at all, and such savings will compound as circuit widths and depths scale.

\subsection{Robustness Under Noise}
We next examine WSBD under noisy, hardware-realistic conditions. For the VQE task with transverse-field Ising Hamiltonians, we simulate errors calibrated from IBM Heron devices, including amplitude/phase damping, depolarization, and readout errors. Results in Tables \ref{tab:vqe_percent_reduction} and \ref{tab:vqe_to_taget_energy_DETAILED} reveal that WSBD retains significant advantages even when every gate and measurement is perturbed. For example, on the 2-qubit, 2-layer VQE problem, WSBD-SGD required 58\% fewer forward passes than SGD, while WSBD-Adam reduced forward passes by more than 80\%. 

These results can be interpreted mathematically. Let $\nabla f(\params)$ denote the true gradient of the cost function in the VQE problem, and let $\varepsilon$ represent a noise channel acting on each circuit evaluation. Under noise, the effective gradient becomes:$
\tilde{\nabla}f(\params) = \nabla f(\params) + \varepsilon(\params),$ where $\varepsilon(\params)$ is a stochastic disturbance induced by the various error sources. Standard optimizers attempt to estimate all components of $\nabla f$, but many of these are small in magnitude and easily overwhelmed by $\varepsilon$. WSBD, by design, biases updates toward parameters with consistently large cumulative gradients. This adaptive focus acts as a form of noise filtering: resources are less likely to be wasted on low-signal parameters, and optimization proceeds along directions where the true signal dominates the noise. The practical consequence is that WSBD remains effective even under the severe error models of near-term devices. In fact, its bias toward “signal-rich” parameters may be even more valuable in noisy settings than in idealized simulations, as evidenced by the divergence between baseline and WSBD performance in Tables \ref{tab:vqe_percent_reduction} and \ref{tab:mnist_accuracy} where the performance improvements were often greater in the noisy problem or the training Figures \ref{fig:samp_vqe_training_curve} and \ref{fig:vqe_all} where WSBD often reached a lower target energy overall.

\subsection{Why Each Component Matters: Ablation Study}

Finally, we analyze the contribution of WSBD’s design choices through ablations (Figure \ref{fig:qnn_combined}). Four variants were considered: Deterministic Block Descent (DBD): freezes the lowest-importance parameters deterministically. Stochastic Block Descent (SBD): freezes parameters randomly without importance scores. Layer-wise WSBD: aggregates importance scores per layer and freezes entire layers. WSBD without reset: retains past importance scores when parameters are reactivated. The results highlight several key insights. First, stochastic selection is crucial: DBD achieves convergence but often plateaus at higher losses, while stochastic freezing avoids premature elimination of parameters that may regain importance later. Second, importance scores provide the main source of WSBD’s advantage; SBD performs similarly to plain SGD, confirming that freezing alone is insufficient without intelligent parameter ranking. Third, parameter-wise granularity is essential for QNNs, where influence is not neatly aligned with layer boundaries. WSBD-SGD was nearly 50\% faster than its layer-wise counterpart, demonstrating that coarse freezing strategies cannot capture the entangled parameter dependencies of quantum circuits. Finally, the reset mechanism prevents stale importance values from dominating; without resets, optimization stagnates on both parity and MNIST tasks. Together, these ablations show that each design element of WSBD is necessary. Far from being an incidental combination of heuristics, WSBD’s success depends on the interplay of stochasticity, fine-grained importance tracking, and adaptive resetting. Finally for the gradient-free optimizers, although Simultaneous Perturbation Stochastic Approximation (SPSA) \citep{spall2002multivariate}, Nelder-Mead \citep{nelder1965simplex}, and Bayesian optimization \citep{jones1998efficient} require fewer circuit evaluations per step, they consistently stagnated on larger models. This suggests that their low-cost updates do not provide enough directional information for challenging optimization landscapes.

\section{CONCLUSION}

In this work, we have introduced Weighted Stochastic Block Descent, a dynamic parameter freezing optimizer designed to directly confront the burdens that hinder QNN training. Our results showcase that by intelligently focusing on a subset of the most influential parameters, WSBD consistently accelerates convergence and achieves superior performance compared to standard optimizers and gradient-free methods. Our ablation studies validate the core principles behind WSBD's design: the efficacy of its stochastic selection process, the use of gradient-based importance scores, the critical need for a granular parameter-wise approach, and the adaptive score-resetting mechanism. These elements, supported by a formal proof of convergence, establish dynamic parameter freezing as a potent and practical strategy for QNN optimization. Beyond these immediate results, this work opens up several exciting future directions. The principles of WSBD could be extended to create hardware-aware optimizers, where freezing is tailored to the noise profiles of specific quantum devices. Furthermore, the ability to make training more tractable may enable researchers to explore novel QNN architectures that are currently considered too deep or wide to train effectively. \textit{Ultimately, by transforming the training process from a brute-force calculation into an intelligent allocation of scarce resources, WSBD represents a significant step toward making QML a practical reality.}

\subsection*{Acknowledgements}
This work was supported by the following NSF grants: CSR-2402328, CAREER-2338457, CSR-2406069, CSR-2323100, and HRD-2225201, as well as by the DOE Office of Science User Facility, supported by the Office of Science of the U.S. Department of Energy under Contract No. DE-AC02-05CH11231, using NERSC award DDR-ERCAP0035598. We also acknowledge support from Rice University and Florida International University.
\bibliographystyle{apalike}
\bibliography{ref}
 % or plainnat, apalike, abbrvnat etc.
% \bibliography{iclr2026_conference.bib}

% \begin{thebibliography}{}
% \setlength{\itemindent}{-\leftmargin}
% \makeatletter\renewcommand{\@biblabel}[1]{}\makeatother
% \bibitem{} J.~Alspector, B.~Gupta, and R.~B.~Allen (1989).
%     \newblock Performance of a stochastic learning microchip.
%     \newblock In D. S. Touretzky (ed.),
%     \textit{Advances in Neural Information Processing Systems 1}, 748--760.
%     San Mateo, Calif.: Morgan Kaufmann.

% \bibitem{} F.~Rosenblatt (1962).
%     \newblock \textit{Principles of Neurodynamics.}
%     \newblock Washington, D.C.: Spartan Books.

% \bibitem{} G.~Tesauro (1989).
%     \newblock Neurogammon wins computer Olympiad.
%     \newblock \textit{Neural Computation} \textbf{1}(3):321--323.
% \end{thebibliography}

%%%%%%%%%%%%%%%%%%%%%%%%%%%%%%%%%%%%%%%%%%%%%%%%%%%%%%%%%%%%
\section*{Checklist}

% %%% BEGIN INSTRUCTIONS %%%
% The checklist follows the references. For each question, choose your answer from the three possible options: Yes, No, Not Applicable.  You are encouraged to include a justification to your answer, either by referencing the appropriate section of your paper or providing a brief inline description (1-2 sentences). 
% Please do not modify the questions.  Note that the Checklist section does not count towards the page limit. Not including the checklist in the first submission won't result in desk rejection, although in such case we will ask you to upload it during the author response period and include it in camera ready (if accepted).

% \textbf{In your paper, please delete this instructions block and only keep the Checklist section heading above along with the questions/answers below.}
% %%% END INSTRUCTIONS %%%

\begin{enumerate}

  \item For all models and algorithms presented, check if you include:
  \begin{enumerate}
    \item A clear description of the mathematical setting, assumptions, algorithm, and/or model. \textbf{Yes.} WSBD  is summarized in Algorithm \ref{alg:wsbd} with proof and assumptions given in Appendix \ref{sec:proof}.
    % [Yes/No/Not Applicable]
    \item An analysis of the properties and complexity (time, space, sample size) of any algorithm.
    \textbf{Yes.} The time complexity for WSBD is given in Section \ref{sec:optim} being $O(|\params|)$.
    \item (Optional) Anonymized source code, with specification of all dependencies, including external libraries. \textbf{Yes.} All source code is provided in the supplementary material with hardware and dependencies specification given in Appendix \ref{app:codeanddataappendix}.
  \end{enumerate}

  \item For any theoretical claim, check if you include:
  \begin{enumerate}
    \item Statements of the full set of assumptions of all theoretical results. \textbf{Yes.} See Appendix \ref{sec:proof}.
    \item Complete proofs of all theoretical results. \textbf{Yes.} L-smoothness and convergence is proven in Appendix \ref{sec:proof}.
    \item Clear explanations of any assumptions. \textbf{Yes.} Standard assumptions for convergence is given in Appendix \ref{app:converence_proof}. 
  \end{enumerate}

  \item For all figures and tables that present empirical results, check if you include:
  \begin{enumerate}
    \item The code, data, and instructions needed to reproduce the main experimental results (either in the supplemental material or as a URL). \textbf{Yes.} All code to reproduce our findings are available on GitHub.
    \item All the training details (e.g., data splits, hyperparameters, how they were chosen). \textbf{Yes.} Section \ref{sec:optim} and Appendix \ref{app:imp_scores} shows the importance score, freezing percentile, and training window used.
    \item A clear definition of the specific measure or statistics and error bars (e.g., with respect to the random seed after running experiments multiple times). \textbf{Yes.} See Section \ref{sec:exp_design} with a variation in convergence less than 3\%.
    \item A description of the computing infrastructure used. (e.g., type of GPUs, internal cluster, or cloud provider). \textbf{Yes.} Both quantum and classical hardware used is reported in Appendix \ref{app:codeanddataappendix}.
  \end{enumerate}

  \item If you are using existing assets (e.g., code, data, models) or curating/releasing new assets, check if you include:
  \begin{enumerate}
    \item Citations of the creator If your work uses existing assets. \textbf{Yes.} MNIST dataset is properly cited.
    \item The license information of the assets, if applicable. \textbf{Not Applicable.}
    \item New assets either in the supplemental material or as a URL, if applicable. \textbf{Not Applicable.}
    \item Information about consent from data providers/curators. \textbf{Not Applicable.}
    \item Discussion of sensible content if applicable, e.g., personally identifiable information or offensive content. \textbf{Not Applicable.}
  \end{enumerate}

  \item If you used crowdsourcing or conducted research with human subjects, check if you include:
  \begin{enumerate}
    \item The full text of instructions given to participants and screenshots. \textbf{Not Applicable.}
    \item Descriptions of potential participant risks, with links to Institutional Review Board (IRB) approvals if applicable. \textbf{Not Applicable.}
    \item The estimated hourly wage paid to participants and the total amount spent on participant compensation. \textbf{Not Applicable.}
  \end{enumerate}

\end{enumerate}

\clearpage
\appendix
\thispagestyle{empty}

% Supplementary material: To improve readability, you must use a single-column format for the supplementary material.
\onecolumn
\aistatstitle{APPENDIX}

\section{THEORETICAL ANALYSIS}
\label{sec:proof}
% In this section we will formally prove QNNs are L-smooth, an assumption we use to prove convergence for WSBD.

% \begin{definition}\label{def:L-smooth}
% A function $f: \mathbb{R}^P \to \mathbb{R}$ is \textit{L-smooth} for a constant $L \ge 0$ if $\forall \boldsymbol{\theta}_1, \boldsymbol{\theta}_2 \in \mathbb{R}^P$, the following inequality holds \cite{nesterov2013introductory}:
% \begin{equation}
%     \|\nabla f(\boldsymbol{\theta}_1) - \nabla f(\boldsymbol{\theta}_2)\|_2 \le L \|\boldsymbol{\theta}_1 - \boldsymbol{\theta}_2\|_2
% \end{equation}
% For a twice-differentiable function, a sufficient condition for L-smoothness is that the spectral norm of its Hessian matrix is globally bounded \cite{nesterov2013introductory}:
% \begin{equation}
%     \|\nabla^2 f(\boldsymbol{\theta})\|_2 \le L
% \end{equation}
% \end{definition}

%%%%%%%%%%%%%%%%%%%%%%%%%%%

In this section, we formally prove that the objective functions of parametrized quantum neural networks (QNNs) are $L$-smooth. Rather than computing the Hessian explicitly or bounding it term by term, our strategy will be to study how the gradient changes along parameter paths. By carefully tracking commutator structures and using only boundedness of the generators and the observable, we arrive at a direct and conceptually transparent proof of smoothness.

\begin{theorem}[Smoothness of QNN objectives]
\label{thm:elem_bound}
Let the objective function be denoted by
\[
f(\params) \;=\; \langle \psi_0 \mid U(\params)^\dagger M U(\params) \mid \psi_0\rangle,
\qquad 
U(\params)=\prod_{j=1}^P U_j(\params_j),\quad
U_j(\params_j)=e^{-i\,\params_j G_j},
\]
where $M$ is a bounded Hermitian observable, and each $G_j$ is a bounded Hermitian generator. Then the gradient $\nabla f(\params)$ is globally Lipschitz continuous, i.e., there exists a constant $L<\infty$ such that
\[
\|\nabla f(\params+h)-\nabla f(\params)\| \;\le\; L\,\|h\|
\qquad\forall\params,h\in\mathbb{R}^P.
\]
\end{theorem}

\begin{proof}
We proceed in steps.

\paragraph{Step 1: Expressing derivatives as commutators.}
Let $M(\params)=U(\params)^\dagger M U(\params)$ denote the Heisenberg-evolved observable. A straightforward calculation yields
\[
\partial_j f(\params) 
= i\,\langle \psi_0 \mid U(\params)^\dagger [M, \tilde G_j(\params)] U(\params)\mid \psi_0\rangle,
\]
where $\tilde G_j(\params)$ is the generator $G_j$ conjugated by part of the circuit. Since conjugation by a unitary preserves operator norm, we always have $\|\tilde G_j(\params)\|=\|G_j\|$.

\paragraph{Step 2: Following the gradient along a path.}
Fix $\params,h\in\mathbb{R}^P$, and consider the interpolation $\params(t)=\params+th$, $t\in[0,1]$. Define
\[
\Phi(t)=\nabla f(\params(t)).
\]
By construction,
\[
\nabla f(\params+h)-\nabla f(\params) = \Phi(1)-\Phi(0) = \int_0^1 \Phi'(t)\,dt.
\]
Hence the Lipschitz property will follow once we show that $\|\Phi'(t)\|$ can be bounded uniformly by a constant multiple of $\|h\|$.

\paragraph{Step 3: Structure of $\Phi'(t)$.}
Differentiating $\Phi_j(t)$ with respect to $t$ produces expressions that are linear combinations of expectation values of \emph{double commutators} of the form
\[
\langle \psi_0 \mid U(\params(t))^\dagger [A,[B,C]] U(\params(t)) \mid \psi_0\rangle,
\]
where $A,B,C$ are drawn from the bounded set $\{M\}\cup \{\tilde G_k(\params(t))\}_{k=1}^P$. The coefficients of these terms depend linearly on the components of $h$.

\paragraph{Step 4: Uniform boundedness.}
By the standard inequality $\|[A,[B,C]]\|\le 4\|A\|\|B\|\|C\|$ and the invariance of operator norm under conjugation, each of these expectation values is uniformly bounded by a constant multiple of $\|h\|$. Importantly, the constant depends only on $\|M\|$ and the norms of the generators $\|G_j\|$, and is independent of $\params$, $h$, or $t$.

Thus, for all $t\in[0,1]$,
\[
\|\Phi'(t)\| \;\le\; C\,\|h\|
\]
for some finite constant $C$ determined solely by the problem data.

\paragraph{Step 5: Conclusion.}
Combining the previous displays gives
\[
\|\nabla f(\params+h)-\nabla f(\params)\|
\le \int_0^1 \|\Phi'(t)\|\,dt
\le C\,\|h\|.
\]
This shows that $\nabla f$ is globally Lipschitz, i.e.\ $f$ is $L$-smooth, with $L=C$.
\end{proof}

In summary, we have shown that QNN objectives are $L$-smooth under the mild assumption that the observable and all generators are bounded. The proof did not require constructing the Hessian or computing explicit constants; instead, it relied on the pathwise behavior of the gradient and the boundedness of commutator structures. This establishes smoothness in a direct and conceptually simple way.

\subsection{Proof of Convergence WSBD-SGD}
\label{app:converence_proof}
This proof of convergence is specifically tailored to the standard gradient descent update rule, thereby formally guaranteeing convergence for the WSBD-SGD variant. While our extensive empirical results demonstrate strong performance and faster convergence for the WSBD-Adam variant across all tasks, a formal convergence proof for adaptive optimizers like Adam is considerably more complex, as it relies on a different set of assumptions to bound the update steps. We believe a formal proof for the adaptive case is achievable and propose its development as a promising direction for future work.

\textbf{Variable Names:} Let $\mathcal{C}^{(t)}$represent the objective function and $\nabla \mathcal{C}^{(t)}$ represent the gradient for the current parameter vector $\params^{(t)}$. Let $p_k^{(t)} = \mathbb{P}(\theta_k \in \mathbb{A}^{(t)} | \params^{(t)})$ be the probability that parameter $\theta_k$ is in the active set. Let $\delta^{(t)}$ be a vector with entries $\delta_k^{(t)} = 1$ if $\theta_k \in\mathbb{A}^{(t)}$ and $0$ if $\theta_k \notin \mathbb{A}^{(t)}$. 

\textbf{Assumptions:}
\begin{itemize}
    \item \textit{Cost Function is L-Smooth}: We prove this property in Theorem \ref{thm:elem_bound}.
    \item  \textit{Strictly Positive Selection Probability}: $p_k^{(t)} \ge p_{\min} > 0$, which is guaranteed as $p_{min} = \frac{\epsilon}{\sum_{i=1}^{|\params|} (|\mathcal{I}_p(\theta_i)| + \epsilon)}$.
    \item \textit{The cost function $\mathcal{C}^{(t)}$ is bounded below optimum} $\mathcal{C}^*$.
\end{itemize}

\begin{proof}
We begin the proof with Assumption 1.
\begin{equation*}
\mathcal{C}^{(t+1)} \le \mathcal{C}^{(t)} + \langle \nabla\mathcal{C}^{(t)}, \params^{(t+1)} - \params^{(t)} \rangle + \frac{L}{2}||\params^{(t+1)} - \params^{(t)}||^2
\end{equation*}
The WSBD update rule is given by: $$\params^{(t+1)} = \params^{(t)} - \eta\delta^{(t)}\nabla\mathcal{C}^{(t)}$$% where $\delta^{(t)}$ is a vector with entries $\delta_k^{(t)} = 1$ if $\theta_k$ is in $\mathbb{A}^{(t)}$ and $0$ otherwise. 
Substituting the update rule gives:
\begin{equation*}
\mathcal{C}^{(t+1)} \le \mathcal{C}^{(t)} - \eta\langle \nabla\mathcal{C}^{(t)}, \delta^{(t)}\nabla\mathcal{C}^{(t)} \rangle + \frac{L\eta^2}{2}||\delta^{(t)}\nabla\mathcal{C}^{(t)}||^2
\end{equation*}

Taking the expectation over the stochasticity of the process (both in the gradient estimation and $\delta^{(t)}$):
\begin{align*}
\mathbb{E}[\mathcal{C}^{(t+1)}] \le \mathbb{E}[\mathcal{C}^{(t)}] - \eta\mathbb{E}[\langle \nabla\mathcal{C}^{(t)}, \delta^{(t)}\nabla\mathcal{C}^{(t)} \rangle]  + \frac{L\eta^2}{2}\mathbb{E}[||\delta^{(t)}\nabla\mathcal{C}^{(t)}||^2]
\end{align*}

We analyze the inner product term. Using the law of total expectation, conditioned on $\params^{(t)}$:
\begin{align*}
\mathbb{E}[\langle \nabla\mathcal{C}^{(t)}, \delta^{(t)}\nabla\mathcal{C}^{(t)} \rangle] 
%&= \mathbb{E}\left[\textstyle\sum_{k=1}^{|\params|} (\nabla_k\mathcal{C}^{(t)})^2 \delta_k^{(t)}\right] \\
&= \mathbb{E}\left[\mathbb{E}\left[\sum_{k=1}^{|\params|} (\nabla_k\mathcal{C}^{(t)})^2 \delta_k^{(t)} \bigg| \params^{(t)}\right]\right]
= \mathbb{E}\left[\sum_{k=1}^{|\params|} (\nabla_k\mathcal{C}^{(t)})^2 \mathbb{E}[\delta_k^{(t)} | \params^{(t)}]\right]
\end{align*}

$\mathbb{E}[\delta_k^{(t)} | \params^{(t)}]$ can be simplified to $p_k^{(t)} \geq p_{min}$, thus:
$$
\mathbb{E}[\langle \nabla\mathcal{C}^{(t)}, \delta^{(t)}\nabla\mathcal{C}^{(t)} \rangle] %\ge \mathbb{E}\left[\textstyle\sum_{k=1}^{|\params|} (\nabla_k\mathcal{C}^{(t)})^2 p_{\min}\right]  \\ = p_{\min}\mathbb{E}[||\nabla\mathcal{C}^{(t)}||^2]
\geq  p_{\min}\mathbb{E}[||\nabla\mathcal{C}^{(t)}||^2]
$$

Next, we bound the final term $||\delta^{(t)}\nabla\mathcal{C}^{(t)}||^2$. Since $\delta_k^{(t)} \in \{0, 1\}$, we have $(\delta_k^{(t)})^2 = \delta_k^{(t)} \le 1$, thus:
\begin{align*}
\mathbb{E}[||\delta^{(t)}\nabla\mathcal{C}^{(t)}||^2] \leq
%\\
%= \mathbb{E}\left[\sum_{k=1}^{|\params|} (\delta_k^{(t)})^2 (\nabla_k\mathcal{C}^{(t)})^2\right] \le \mathbb{E}\left[\sum_{k=1}^{|\params|} (\nabla_k\mathcal{C}^{(t)})^2\right] \\ = 
\mathbb{E}[||\nabla\mathcal{C}^{(t)}||^2]
\end{align*}

Substituting these bounds back into the main inequality:
\begin{align*}
\mathbb{E}[\mathcal{C}^{(t+1)}] &\le \mathbb{E}[\mathcal{C}^{(t)}] - \eta p_{\min}\mathbb{E}[||\nabla\mathcal{C}^{(t)}||^2] + \frac{L\eta^2}{2}\mathbb{E}[||\nabla\mathcal{C}^{(t)}||^2]
= \mathbb{E}[\mathcal{C}^{(t)}] - \eta\left(p_{\min} - \frac{L\eta}{2}\right)\mathbb{E}[||\nabla\mathcal{C}^{(t)}||^2]
\end{align*}

For convergence, we require the learning rate $\eta$ to be chosen such that $p_{\min} > \frac{L\eta}{2}$. Let $C_\eta = \eta(p_{\min} - \frac{L\eta}{2})$. Rearranging the inequality:
\begin{equation*}
C_\eta \mathbb{E}[||\nabla\mathcal{C}^{(t)}||^2] \le \mathbb{E}[\mathcal{C}^{(t)}] - \mathbb{E}[\mathcal{C}^{(t+1)}]
\end{equation*}

Now, we sum this inequality from $t=0$ to $T-1$:
\begin{align*}
\sum_{t=0}^{T-1} C_\eta \mathbb{E}[||\nabla\mathcal{C}^{(t)}||^2] &\le \sum_{t=0}^{T-1} (\mathbb{E}[\mathcal{C}^{(t)}] - \mathbb{E}[\mathcal{C}^{(t+1)}])
\end{align*}

This simplifies to:
\begin{align*}
C_\eta \sum_{t=0}^{T-1} \mathbb{E}[||\nabla\mathcal{C}^{(t)}||^2] &\le \mathbb{E}[\mathcal{C}^{(0)}] - \mathbb{E}[\mathcal{C}^{(T)}]
\end{align*}

Following assumption 3, we have $\mathbb{E}[\mathcal{C}^{(T)}] \ge \mathcal{C}^*$.
\begin{equation*}
\sum_{t=0}^{T-1} \mathbb{E}[||\nabla\mathcal{C}^{(t)}||^2] \le \frac{\mathcal{C}^{(0)} - \mathcal{C}^*}{C_\eta}
\end{equation*}

Dividing by $T$ and taking the limit as $T \to \infty$:
\begin{equation*}
\lim_{T\to\infty} \frac{1}{T}\sum_{t=0}^{T-1} \mathbb{E}[||\nabla\mathcal{C}^{(t)}||^2] \le \lim_{T\to\infty} \frac{\mathcal{C}^{(0)} - \mathcal{C}^*}{T C_\eta} = 0
\end{equation*}

Thus, the expected squared norm of the gradient converges to zero on average, which completes the proof of convergence for WSBD. Convergence is therefore guaranteed for any importance metric $\mathcal{I}_p$ that ensures $p_{min} > 0$, training window $\tau > 0$, and for any freezing threshold $\lambda_f$ that ensures $|\mathbb{A}| \geq 1$, making WSBD highly tunable.
\end{proof}

\subsection{Proof of Convergence for WSBD-ADAM (AMSGrad Variant)}
\label{app:proof_wsbd_adam}

While standard Adam lacks general convergence guarantees in certain non-convex settings \citep{reddi2019convergence}, its AMSGrad variant restores convergence by enforcing non-increasing adaptive learning rates. As AMSGrad and Adam behave similarly in practice, we provide a convergence proof for \textbf{WSBD-AMSGrad}, which serves as a  proxy for the WSBD-Adam optimizer used in the main text.

\paragraph{Update Rule.}
Let $g_t = \nabla \mathcal{C}^{(t)}$ be the (stochastic) gradient at iteration $t$, and let $\delta^{(t)} \in \{0,1\}^d$ be the coordinate-wise WSBD freezing mask. The masked adaptive moments evolve as:
\begin{align*}
    m_t &= \beta_1 m_{t-1} + (1-\beta_1)(\delta^{(t)} \odot g_t), \\
    v_t &= \beta_2 v_{t-1} + (1-\beta_2)(\delta^{(t)} \odot g_t)^2, \\
    \hat{v}_t &= \max(\hat{v}_{t-1}, v_t) \qquad \text{(AMSGrad correction)}, \\
    \theta^{(t+1)} &= \theta^{(t)} - \eta_t \, \delta^{(t)} \odot \frac{m_t}{\sqrt{\hat{v}_t} + \epsilon}.
\end{align*}
Let $H_t$ denote the diagonal preconditioner with entries $(H_t)_{kk} = 1/(\sqrt{\hat{v}_{t,k}} + \epsilon)$.

\paragraph{Additional Assumptions.}
In addition to the smoothness and bounded-below assumptions from Appendix~\ref{app:converence_proof}, we assume:
\begin{itemize}
    \item \textbf{Bounded Gradients:} $\| g_t \|_\infty \le G_\infty$. This implies
    \[
    0 < c_\ell \;\le\; (H_t)_{kk} \;\le\; c_u < \infty.
    \]
    \item \textbf{Decaying Step Size:} $\sum_{t} \eta_t = \infty$ and $\sum_t \eta_t^2 < \infty$.
    \item \textbf{Mask Selection Probability:} Each coordinate is active with probability at least $p_{\min} > 0$:
    \[
        \mathbb{P}\!\left[(\delta^{(t)})_k = 1\right] \ge p_{\min}.
    \]
    The mask is independent of stochastic gradient noise conditioned on past iterates.
\end{itemize}

\begin{proof}
By $L$-smoothness of $\mathcal{C}$,
\begin{equation}\label{eq:smoothness}
\mathcal{C}^{(t+1)} \le \mathcal{C}^{(t)} +
\big\langle \nabla\mathcal{C}^{(t)}, \theta^{(t+1)} - \theta^{(t)} \big\rangle
+ \frac{L}{2}\big\|\theta^{(t+1)} - \theta^{(t)}\big\|^2.
\end{equation}
Substituting the update $\theta^{(t+1)} - \theta^{(t)} = - \eta_t\, \delta^{(t)} \odot H_t m_t$, we obtain
\begin{align*}
\mathcal{C}^{(t+1)}
&\le
\mathcal{C}^{(t)}
- \eta_t \big\langle \nabla\mathcal{C}^{(t)},\, \delta^{(t)} \odot H_t m_t \big\rangle
+ \frac{L\eta_t^2}{2}\big\|\delta^{(t)} \odot H_t m_t\big\|^2.
\end{align*}

\paragraph{Bounding the Quadratic Term.}
Since $\|g_t\|_\infty \le G_\infty$ and AMSGrad ensures monotone $\hat{v}_t$, standard arguments \citep{reddi2019convergence} imply that $m_t$ and $H_t m_t$ are uniformly bounded. Thus, there exists $K>0$ such that
\[
\big\|\delta^{(t)} \odot H_t m_t\big\|^2 \le K,
\]
and therefore
\[
\frac{L\eta_t^2}{2}\big\|\delta^{(t)} \odot H_t m_t\big\|^2 \le \frac{LK}{2}\eta_t^2.
\]

\paragraph{Analyzing the Descent Term.}
Taking conditional expectation over both gradient noise and mask randomness:
\[
\mathbb{E}\!\left[
\big\langle \nabla\mathcal{C}^{(t)},\, \delta^{(t)} \odot H_t m_t \big\rangle
\mid \mathcal{F}_t
\right]
=
\Big\langle \nabla\mathcal{C}^{(t)},\,
\mathbb{E}[\delta^{(t)} \mid \mathcal{F}_t] \odot H_t m_t \Big\rangle.
\]
Since each coordinate satisfies $\mathbb{E}[(\delta^{(t)})_k \mid \mathcal{F}_t] \ge p_{\min}$,
\[
\mathbb{E}[\delta^{(t)} \mid \mathcal{F}_t] \succeq p_{\min} I,
\]
and thus
\[
\mathbb{E}\!\left[
\big\langle \nabla\mathcal{C}^{(t)},\, \delta^{(t)} \odot H_t m_t \big\rangle
\right]
\ge
p_{\min}\,
\mathbb{E}\!\left[
\big\langle \nabla\mathcal{C}^{(t)},\, H_t m_t \big\rangle
\right].
\]

\paragraph{Invoking the AMSGrad Descent Lemma.}
Non-convex AMSGrad analysis (e.g.\ \citep{reddi2019convergence}) guarantees constants $c>0$ and $B<\infty$ such that
\[
\sum_{t=0}^{T-1} \eta_t\,
\mathbb{E}\!\left[
\big\langle \nabla\mathcal{C}^{(t)},\, H_t m_t\big\rangle
\right]
\;\ge\;
c\sum_{t=0}^{T-1} \eta_t\, \mathbb{E}\|\nabla\mathcal{C}^{(t)}\|^2 - B.
\]
Combining with the bound above:
\[
\sum_{t=0}^{T-1} \eta_t\,
\mathbb{E}\!\left[
\big\langle \nabla\mathcal{C}^{(t)},\, \delta^{(t)} \odot H_t m_t\big\rangle
\right]
\;\ge\;
p_{\min}\!\left(
c\sum_{t=0}^{T-1} \eta_t\, \mathbb{E}\|\nabla\mathcal{C}^{(t)}\|^2 - B
\right).
\]

\paragraph{Final Convergence Argument.}
Taking expectations in \eqref{eq:smoothness}, summing from $t=0$ to $T-1$, and using the bounds above yields:
\[
p_{\min} c
\sum_{t=0}^{T-1} \eta_t\,\mathbb{E}\|\nabla\mathcal{C}^{(t)}\|^2
\;\le\;
\mathcal{C}^{(0)} - \mathcal{C}^* + p_{\min} B + \frac{LK}{2}\sum_{t=0}^{\infty} \eta_t^2
\;=\; M < \infty.
\]
Because $\sum_t \eta_t = \infty$, finiteness of the weighted sum implies
\[
\liminf_{t\to\infty} \mathbb{E}\|\nabla\mathcal{C}^{(t)}\|^2 = 0,
\]
i.e.\ the WSBD-AMSGrad iterates converge to a stationary point in expectation.
\end{proof}

\subsection{General Convergence Framework for WSBD}
\label{app:general_convergence}

Having established convergence for SGD and AMSGrad, we now extend the theoretical guarantees to a broad class of optimizers. The key observation is that WSBD acts as a stochastic coordinate mask that scales the expected descent direction of the base optimizer without amplifying its variance. Under mild and standard assumptions on the base optimizer $\mathcal{O}$, this suffices to preserve convergence.

\paragraph{Generalized Update Rule.}
Let $\mathcal{O}$ be any optimizer that, at step $t$, generates a proposed update vector $u_t \in \mathbb{R}^{|\theta|}$ based on the gradient $\nabla\mathcal{C}^{(t)}$ and internal states (e.g., momentum buffers). The standard update would be $\theta^{(t+1)} = \theta^{(t)} - \eta_t u_t$.
The \textbf{WSBD-augmented update} applies the freezing mask $\delta^{(t)} \in \{0,1\}^{|\theta|}$ to this vector:
\begin{equation}
    \theta^{(t+1)} = \theta^{(t)} - \eta_t (\delta^{(t)} \odot u_t)
\end{equation}

\paragraph{Assumptions on Optimizer $\mathcal{O}$.}
Let $\mathcal{F}_t$ denote the filtration generated by all randomness up to step $t$ (gradients, mini-batches, internal states). We assume:
\begin{enumerate}
    \item \textbf{L-Smoothness:} The objective function $\mathcal{C}$ is $L$-smooth and bounded below by $\mathcal{C}^*$.
    \item \textbf{Descent Condition:} The base optimizer produces a valid descent direction in expectation. Specifically, there exists a constant $c_1 > 0$ such that:
    \begin{equation} \label{eq:gen_descent}
        \mathbb{E}[\langle \nabla\mathcal{C}^{(t)}, u_t \rangle \mid \mathcal{F}_t] \ge c_1 \|\nabla\mathcal{C}^{(t)}\|^2
    \end{equation}
    \item \textbf{Bounded Update Magnitude:} The magnitude of the update vector is bounded. There exists $K > 0$ such that:
    \begin{equation}
        \mathbb{E}[\|u_t\|^2 \mid \mathcal{F}_t] \le K
    \end{equation}
    \item \textbf{Mask Independence and Minimum Probability:} The mask draw satisfies $\mathbb{E}[(\delta^{(t)})_k] = p_k^{(t)} \ge p_{\min} > 0$. Furthermore, the random draw of the current mask $\delta^{(t)}$ is conditionally independent of the stochastic noise in $u_t$ given the history $\mathcal{F}_t$.
\end{enumerate}

\begin{proof}
We start with the $L$-smoothness inequality:
\begin{equation*}
\mathcal{C}^{(t+1)} \le \mathcal{C}^{(t)} - \eta_t \langle \nabla\mathcal{C}^{(t)}, \delta^{(t)} \odot u_t \rangle + \frac{L\eta_t^2}{2} \|\delta^{(t)} \odot u_t\|^2
\end{equation*}

\textbf{1. Bounding the Descent Term:}
We analyze the expected inner product using the Law of Total Expectation. Due to the conditional independence of the mask draw $\delta^{(t)}$ and the update $u_t$:
\begin{align*}
\mathbb{E}[\langle \nabla\mathcal{C}^{(t)}, \delta^{(t)} \odot u_t \rangle] 
&= \mathbb{E}\left[ \sum_{k} (\nabla\mathcal{C}^{(t)})_k (u_t)_k \mathbb{E}[\delta_k^{(t)} \mid \mathcal{F}_t] \right] \\
&\ge p_{\min} \mathbb{E}[\langle \nabla\mathcal{C}^{(t)}, u_t \rangle]
\end{align*}
Applying the Descent Condition (Assumption 2):
\begin{equation}
\mathbb{E}[\langle \nabla\mathcal{C}^{(t)}, \delta^{(t)} \odot u_t \rangle] \ge p_{\min} c_1 \mathbb{E}[\|\nabla\mathcal{C}^{(t)}\|^2]
\end{equation}

\textbf{2. Bounding the Second-Order Term:}
Since $\delta_k^{(t)} \in \{0,1\}$, element-wise multiplication by the mask is a contraction in the Euclidean norm ($\|\delta^{(t)} \odot u_t\|^2 \le \|u_t\|^2$). Using Assumption 3:
\begin{equation}
\mathbb{E}[\|\delta^{(t)} \odot u_t\|^2] \le \mathbb{E}[\|u_t\|^2] \le K
\end{equation}

\textbf{3. Final Convergence Argument:}
Substituting these into the smoothness bound and taking full expectations:
\begin{equation*}
\mathbb{E}[\mathcal{C}^{(t+1)}] \le \mathbb{E}[\mathcal{C}^{(t)}] - \eta_t p_{\min} c_1 \mathbb{E}[\|\nabla\mathcal{C}^{(t)}\|^2] + \frac{L\eta_t^2}{2} K
\end{equation*}
Summing from $t=0$ to $T-1$:
\begin{equation*}
p_{\min} c_1 \sum_{t=0}^{T-1} \eta_t \mathbb{E}[\|\nabla\mathcal{C}^{(t)}\|^2] \le \mathcal{C}^{(0)} - \mathcal{C}^* + \frac{LK}{2} \sum_{t=0}^{T-1} \eta_t^2
\end{equation*}
Given standard step-size conditions ($\sum \eta_t = \infty, \sum \eta_t^2 < \infty$), taking the limit as $T \to \infty$ implies:
\begin{equation*}
\liminf_{t\to\infty} \mathbb{E}[\|\nabla\mathcal{C}^{(t)}\|^2] = 0
\end{equation*}
Thus, WSBD-$\mathcal{O}$ converges to a stationary point, provided the base optimizer $\mathcal{O}$ satisfies the descent and boundedness conditions.
\end{proof}
\section{IMPORTANCE SCORE METRICS AND HYPERPARAMETER TUNING}
\label{app:imp_scores}

\begin{figure*}[h]
    \centering
    \begin{subfigure}[b]{0.32\textwidth}
        \includegraphics[width=\linewidth]{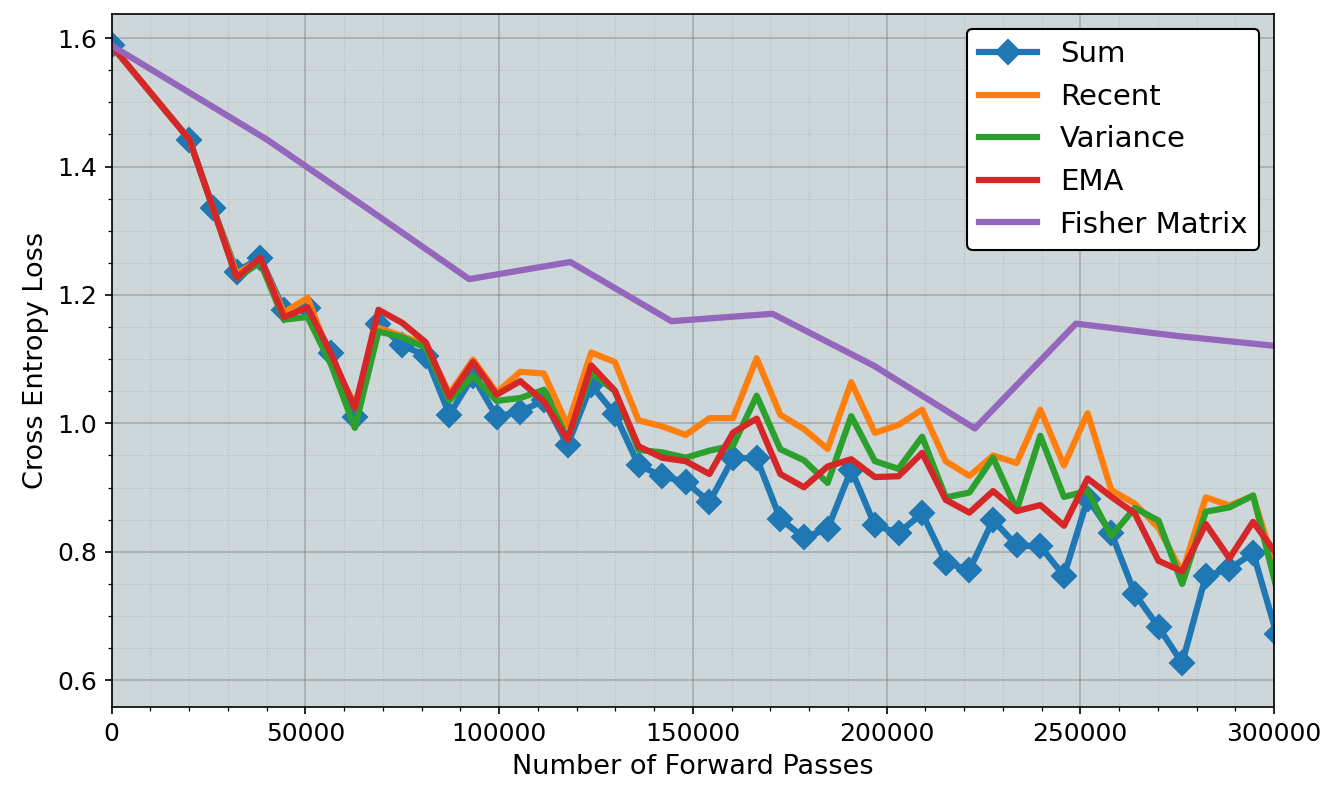}
        \caption{Importance Scores}
        \label{fig:imp_scores}
    \end{subfigure}% <--- NO BLANK LINE
    \hfill
    \begin{subfigure}[b]{0.32\textwidth}
        \includegraphics[width=\linewidth]{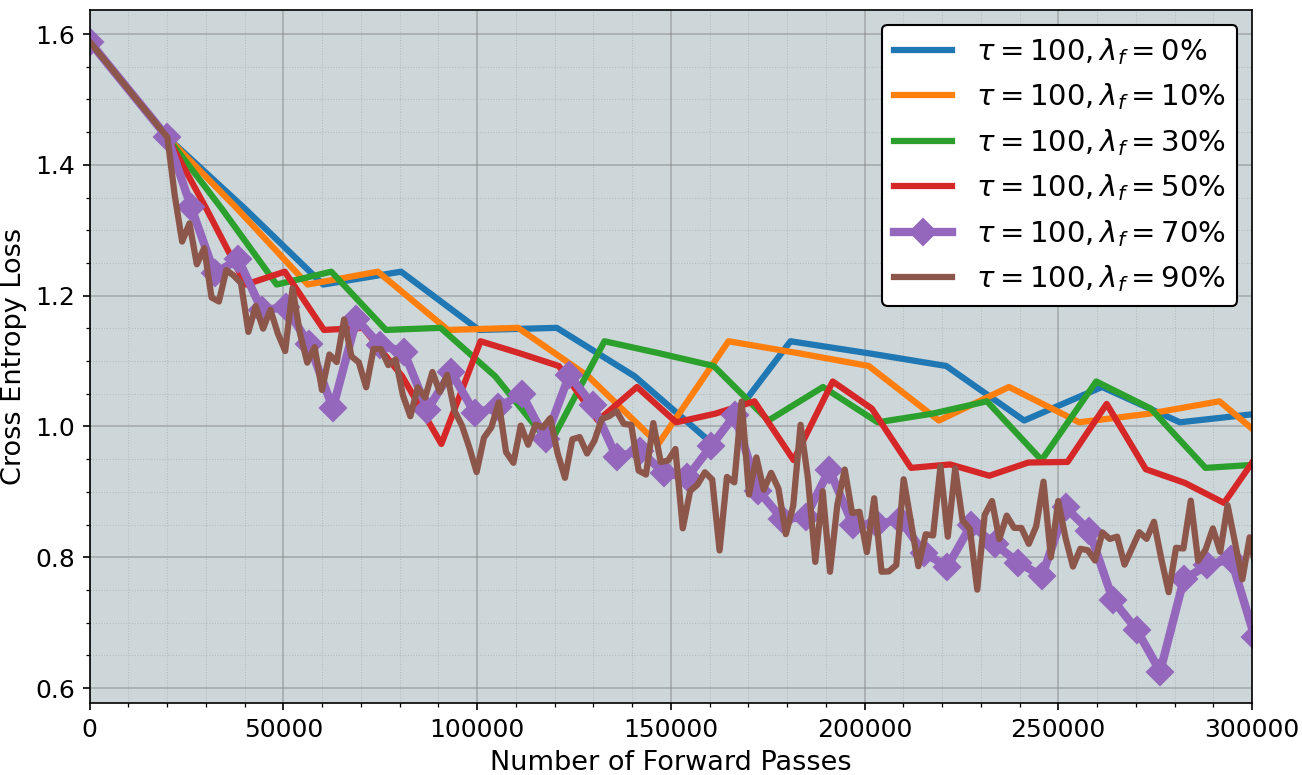}
        \caption{Freeze Thresholds ($\lambda_f$)}
        \label{fig:freeze_t}
    \end{subfigure}% <--- NO BLANK LINE
    \hfill
    \begin{subfigure}[b]{0.32\textwidth}
        \includegraphics[width=\linewidth]{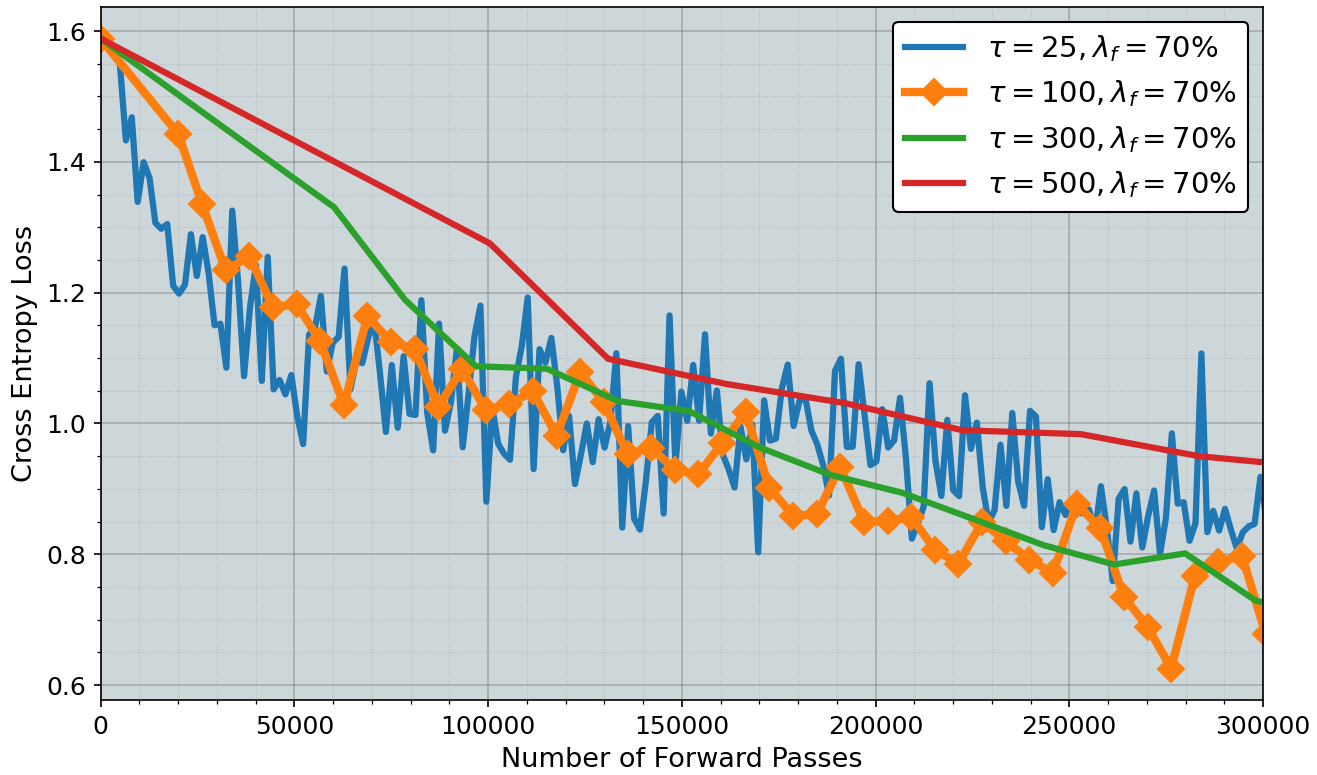}
        \caption{Window Sizes ($\tau$)}
        \label{fig:time_t}
    \end{subfigure}
    
    \vspace{-3mm} % Optional: Adjust vertical space before the main caption
    
    \caption{Hyperparameter tuning for WSBD. Each sub figure shows how the choice in importance score, freezing threshold and training window affects the optimization process. This was done for a 10 qubit, 5 layer QNN on the MNIST problem.}
    \label{fig:hyperparameter_tuning}
    
    \vspace{-5mm} % Optional: Adjust vertical space after the figure
\end{figure*}

To identify the most effective metric for our purposes, we investigated several candidates:

\textbf{Sum of Gradients (Sum):} This metric accumulates the absolute sum of gradients over the window. A large value suggests a consistent and strong impact on the cost function.
\begin{equation}
\mathcal{I}_p(\theta_k) = \Big|\sum_{t=1}^{\tau} \frac{\partial \mathcal{C}(\params^{(t)}, \hat{x_t})}{\partial \theta_k}\Big| + \epsilon
\end{equation}

\textbf{Most Recent Gradient (Recent):} This metric uses only the final gradient (most recent influence) in the window.
\begin{equation}
\mathcal{I}_p(\theta_k) = \Big|\frac{\partial \mathcal{C}(\params^{(\tau)}, \hat{x_\tau})}{\partial \theta_k}\Big| + \epsilon
\end{equation}

\textbf{Gradient Variance (Variance):} This metric gauges the consistency of a parameter's gradient throughout the training window. A low variance suggests a stable, though not necessarily large, gradient. It is computed using Welford's online algorithm \cite{welford1962note}, an efficient single-pass method. 
\begin{equation}
\mathcal{I}_p(\theta_k) = \Big|\text{Var}_{t\in [1, \tau]}\Big(\frac{\partial \mathcal{C}(\params^{(t)}, \hat{x_t})}{\partial \theta_k}\Big)\Big| + \epsilon
\end{equation}

\textbf{Exponential Moving Average (EMA):} This provides a smoothed average of gradients, giving more weight to recent updates. It requires an extra hyperparameter, the decay rate $\beta$. First, we define the exponential moving average of the gradient, $m_k^{(t)}$ or a specific parameter $\theta_k$ at training step t:
\begin{equation}
\mathcal{I}_p(\theta_k)^{(t)} = \beta \mathcal{I}_p(\theta_k)^{(t-1)} + (1-\beta)\frac{\partial \mathcal{C}(\params^{(t)}, \hat{x_t})}{\partial \theta_k}
\end{equation}
The final importance score is then given by:
\begin{equation}
\mathcal{I}_p(\theta_k)^{(t)} = |\mathcal{I}_p(\theta_k)^{(t)}| + \epsilon
\end{equation}

\textbf{Diagonal Fisher (Fisher Matrix):} This is approximated by the average of the squared gradients and is often used as a proxy for the second-order derivative information.
\begin{equation}
\mathcal{I}_p(\theta_k) = \frac{1}{\tau} \sum_{t=1}^\tau \Big(\frac{\partial \mathcal{C}(\params^{(t)}, \hat{x_t})}{\partial \theta_k}\Big)^2 + \epsilon
\end{equation}

Note that a small constant $\epsilon$ (e.g., $10^{-8}$) is added to each score ensuring $\forall \theta_k \in \params, \mathcal{I}_p(\theta_k) > 0$ which will be crucial for proving convergence for WSBD (Sec. \ref{sec:proof}). To select the optimal importance score for WSBD, we compared the five metrics. Each metric was used to train an identical 10-qubit, 5-layer QNN on the MNIST problem, with the performance comparison shown in Fig. \ref{fig:imp_scores}. We make multiple observations. %While the Fisher Matrix is a theoretically powerful metric, it proved less efficient due to the additional forward passes needed to compute the full matrix.  The choice of the importance score metric, \(\mathcal{I}_p\), is central to WSBD's intelligence, as it defines how parameter `influence' is quantified. While most of the metrics we tested outperformed standard optimizers, the \textbf{Sum of Gradients} proved to be the most robust and effective choice.

\begin{itemize}
    \item The \textbf{Fisher Matrix} metric, while theoretically powerful, showed poor performance as the diagonal elements is likely not enough second order information for the optimization process and parameters influence each other.
    \item The \textbf{Gradient Variance} metric was found to be ill-suited for QNNs due to the barren plateau problem. A parameter with a consistently near-zero gradient would have a low variance, but this is a sign of it being stuck on a plateau, not a reliable indicator of its true importance.
    \item The \textbf{Exponential Moving Average (EMA)} performed similarly to the Sum of Gradients but was less effective and required tuning an additional hyperparameter, adding unnecessary complexity.
    \item The \textbf{Sum of Gradients} metric provides a straightforward and highly effective measure of a parameter's impact. It correctly identifies parameters with strong, consistent gradients as influential while effectively flagging those on barren plateaus as unimportant.
\end{itemize}

This makes it the superior choice for navigating the challenges of QNN training. Future work could involve exploring other importance metrics, such as novel hardware-aware metrics that incorporate factors like device-specific error rates to make the optimizer more robust in noisy settings.

\subsection{Hyperparameter Tuning and its Effect on the Optimization Landscape}
\label{sec:appendix_hyerparam}

The WSBD optimizer has three primary hyperparameters: the importance score metric (\(\mathcal{I}_p\)), the training window size (\(\tau\)), and the freeze percentile (\(\lambda_f\)). While our proof of convergence (Appendix \ref{app:converence_proof}) guarantees that the optimizer will converge for any valid setting of these parameters (i.e., $\mathcal{I}_p(\theta)_k > 0 \text{ }\forall \theta_k\text{}, \text{ } \tau > 0, \lambda_f < 100\%$), their values directly influence the training dynamics and overall efficiency. This section explains the practical trade-offs involved with tuning each hyperparameter. We tune WSBD using grid search on two parameters, the freezing threshold $\lambda_f$ and the time window size $\tau$, using the MNIST task on a 10-qubit 5-layer model. First, we fix $\tau = 100$ and vary $\lambda_f$ across $\{0\%, 10\%, 30\%, 50\%, 70\%, 90\%\}$. Then we fix $\lambda_f = 70\%$ and vary $\tau$ across $\{25, 100, 300, 500\}$. The best performance is obtained with $\lambda_f = 70\%$ and $\tau = 100$. This supports the hypothesis that freezing a large fraction of less important parameters accelerates convergence and that shorter time windows help the optimizer adapt more effectively. We notice that these hyperparameter values are generally stable with good performance improvements for $\tau$ being between 100-500 and $\lambda_f$ being between 30-90\%. %See Appendix for a detailed comparison of importance metrics and hyperparameters \ref{app:imp_scores}.

\subsubsection{Training Window Size (\(\tau\))}
The training window, \(\tau\), dictates the number of training steps for which the active set of parameters remains fixed and being trained. This hyperparameter controls the balance between focused optimization and algorithmic adaptivity.

\begin{itemize}
    \item A \textbf{high \(\tau\) value} means the active set is trained for a longer duration. This allows the optimizer to thoroughly explore the subspace defined by the active parameters. However, if a suboptimal set is chosen, it can lead to stagnation before the algorithm has a chance to re-evaluate and select a more influential set of parameters. Additionally the importance of a parameter might stagnate before the window is complete. 
    \item A \textbf{low \(\tau\) value} results in frequent re-shuffling of the active and frozen sets. This enhances the algorithm's adaptivity, allowing it to quickly pivot its focus. However, excessively frequent updates may prevent the optimizer from training any single parameter long enough to accurately gauge its true influence, and the overhead of constantly re-calculating importance scores can diminish computational gains.
\end{itemize}

Our experiments found that values between \textbf{100-300} were optimal, providing a sweet spot that balances deep training of an active set with the flexibility to adapt to the changing optimization landscape.

\subsubsection{Freeze Percentile (\(\lambda_f\))}
The freeze percentile, \(\lambda_f\), determines the aggressiveness of the resource allocation strategy by setting the proportion of parameters to be frozen.

\begin{itemize}
    \item A \textbf{high \(\lambda_f\) value} results in a smaller active set. This offers the greatest computational savings, as fewer forward passes are required for gradient estimation at each step. However, an overly aggressive percentile risks freezing parameters that may become important later in training, potentially hindering the optimizer's ability to explore the full parameter space.
    \item A \textbf{low \(\lambda_f\) value} makes the optimizer behave more like a standard gradient-based method. While this ensures all parameters are explored more frequently, it diminishes the primary benefit of WSBD, as the number of forward passes remains high.
\end{itemize}

Our empirical results showed that a fairly aggressive freeze percentile of \textbf{70\% performed best}. This suggests that at any given stage of training, a relatively small subset of parameters is responsible for the most significant progress, validating WSBD's core hypothesis. Additionally as WSBD enables exploration through its reset mechanism and stochastic selection the high freezing percentile is well justified.S

\section{REPRODUCIBILITY}
\label{app:codeanddataappendix}
\textit{To ensure full reproducibility and encourage adoption and extension by the Quantum Machine Learning (QML) community, all source code, datasets, and analysis scripts used to generate the results presented in this study will be made publicly available on GitHub upon publication.}
\subsection{Software Environment and Dependencies}
\label{sec:appendix_software}

Our implementation is built upon Python (v3.12.5) and leverages several standard, open-source libraries for scientific computing and machine learning.

\begin{itemize}
    \item \textbf{Core QML Framework:} The Quantum Neural Networks (QNNs) were constructed, simulated, and trained using the \texttt{PennyLane} library (v 0.40.0) and training using non-gradient optimizers used \texttt{scikit\_optimize} (v 0.10.2).
    \item \textbf{Data Handling and Processing:} We utilized \texttt{NumPy} (v 2.3.2) for numerical operations and \texttt{Scikit-learn} (v 1.7.1) for sourcing and preprocessing the MNIST dataset. \texttt{Pandas} (v 2.3.1) was used for organizing experimental data.
    \item \textbf{Visualization:} All plots and figures were generated with \texttt{Matplotlib} (v 3.9.2). The training progress was monitored using \texttt{tqdm} (v 4.66.5).
\end{itemize}

A complete \texttt{requirements.txt} file will be included in the root of our public repository to allow for one-step replication of the software environment.

\subsection*{Hardware Configuration and Quantum Computer Used}
The experiments were executed on a high-performance computing server equipped with dual Intel(R) Xeon(R) Gold 6258R CPUs, 1.5 TB of RAM, and eight NVIDIA A100 (40GB) GPUs.

While this information is provided for completeness, our primary performance metric---the number of forward passes required for convergence---is a \textbf{hardware-independent} measure of algorithmic efficiency. The conclusions drawn from our results are therefore transferable across different computing infrastructures.

For the quantum hardware wall-clock experiment we utilized the IBM quantum platform \texttt{ibm\_kingston} quantum computer which has the \texttt{Heron r2} quantum processing unit. We ran 100 random circuits for the different QNN sizes with $\mathcal{M} = 1000$ shots. We used these error rates for the VQE training problem.

\subsection{Experimental Protocol and Control of Randomness}
\label{sec:appendix_randomness}

\textbf{Dataset Integrity:} To ensure a fair and rigorous comparison between all optimizers, we implemented strict controls over stochastic elements. The training and testing datasets for both the MNIST and Parity problems are static and will be provided in our repository. For every experiment, all optimizers were fed data points in the exact same sequence, eliminating any performance variations due to random data shuffling.

\textbf{Parameter Initialization:} For each QNN architecture, the initial variational parameters were randomly generated once and then saved. This fixed set of starting parameters was used for all optimizers being tested on that architecture, guaranteeing an identical starting point for every experimental run.

\textbf{Stochasticity in WSBD:} The only intentionally stochastic component in our experiments was the parameter selection mechanism within the WSBD, SBD, and L-WSBD optimizers, which is a core feature of their design. We found that despite this stochasticity, the performance and convergence behavior of WSBD remained highly stable and consistent across multiple independent runs ($\sim$5), validating the robustness of our approach.

\subsection*{Access and Usage}
The codebase is structured to enable straightforward replication of our findings. The code is organized into two primary directories corresponding to the experimental problems: one for MNIST and one for Parity. Within each directory, we provide a separate Python script for each optimizer (e.g., \texttt{WSBD-SGD.py}, \texttt{SGD.py}, \texttt{DBD.py}, etc.). To reproduce a specific result, the QNN architecture (the number of qubits and layers) can be configured directly within the relevant script. The experiment is then launched by executing that specific file. For instance, to run the 10-qubit, 5-layer Parity problem with WSBD-SGD, a user would modify the parameters inside the \texttt{WSBD-SGD.py} file in the \texttt{/XOR/} folder and then run the script. Detailed instructions for setup and execution will be available in the \texttt{README.md} file in the project's root directory.

\section{ADDITIONAL RESULTS}
This section includes detailed training plots and tables related to the comparison of optimizers. 

\begin{figure}[!h]
    \centering
    \begin{subfigure}{\linewidth}
        \centering
        \includegraphics[width=\linewidth]{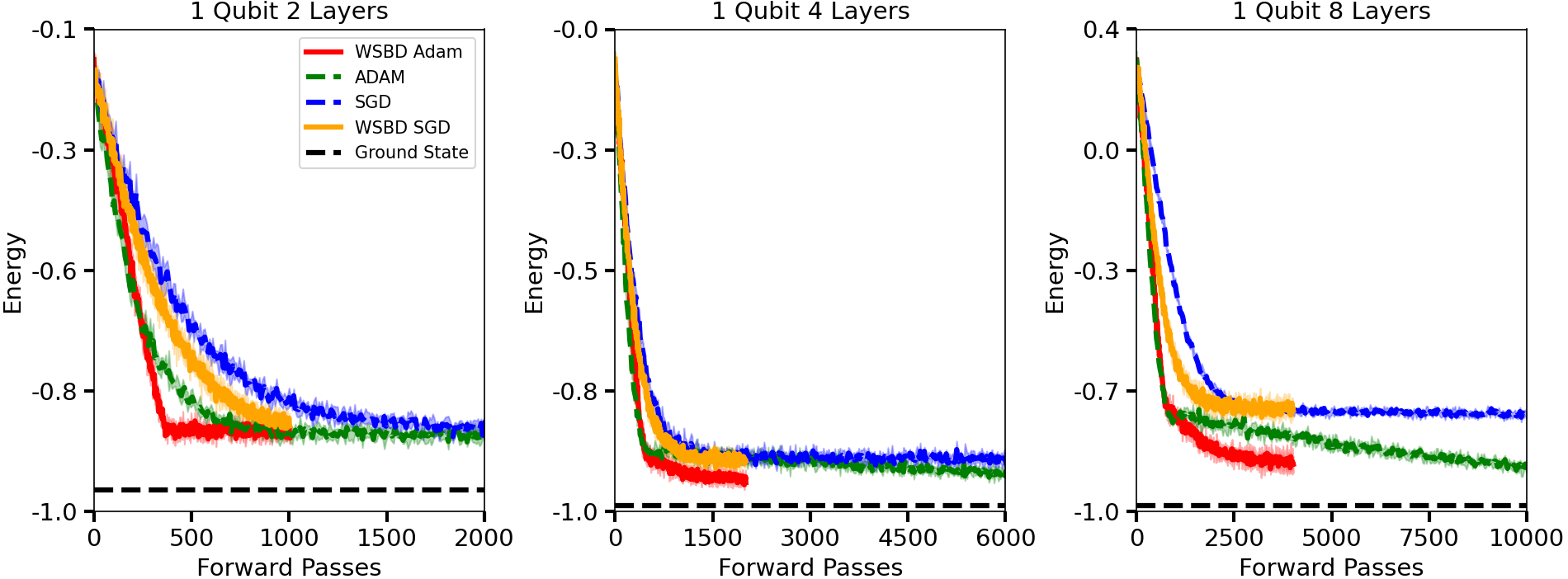}
        % \caption{Caption for 1q}
        \label{fig:vqe_1q}
    \end{subfigure}

    \vspace{1em} % optional space between subfigures

    \begin{subfigure}{\linewidth}
        \centering
        \includegraphics[width=\linewidth]{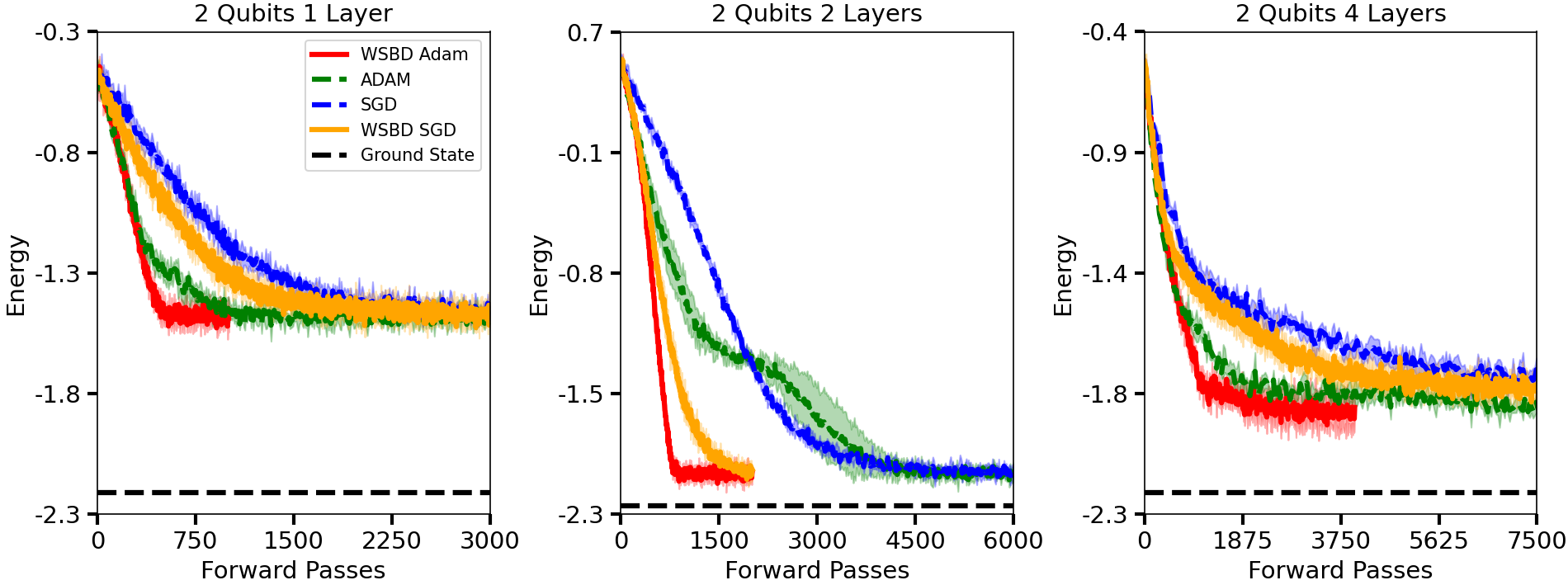}
        % \caption{Caption for 2q}
        \label{fig:vqe_2q}
    \end{subfigure}

    \vspace{1em}

    \begin{subfigure}{\linewidth}
        \centering
        \includegraphics[width=\linewidth]{Figures2/vqe_4q.png}
        % \caption{Caption for 4q}
        \label{fig:vqe_4q}
    \end{subfigure}

    \caption{Training curves for the ground state energy problem using Adam, SGD and their WSBD counterpart optimizers. The black dotted line represents the ground state energy each optimizer aims to reach (closer is better). The models were trained in realistic noisy environments until convergence was reached.}
    \label{fig:vqe_all}
\end{figure}
\begin{figure*}[h]
    \centering
    \begin{subfigure}[t]{0.32\textwidth}
        \centering
        \includegraphics[width=\textwidth]{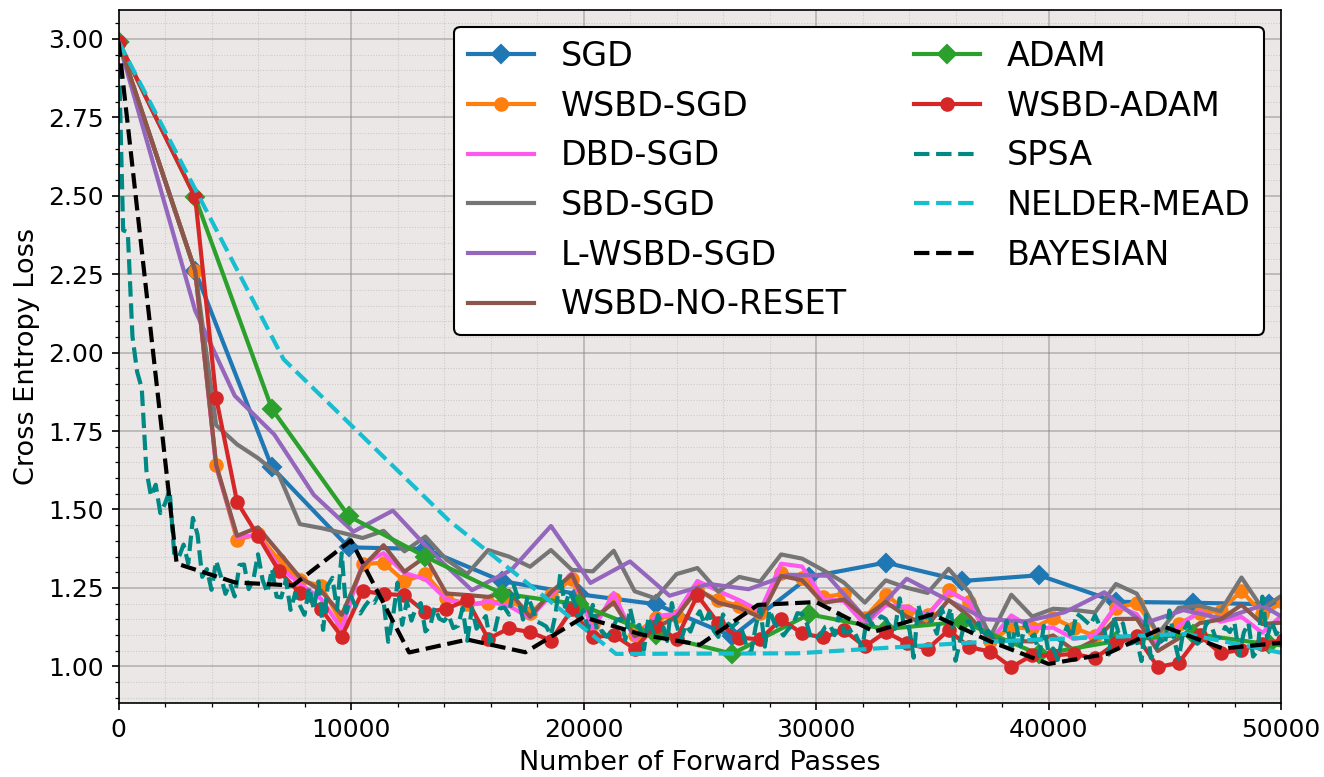}
        \caption{MNIST 4 Qubits 2 Layers}
        \label{fig:training_4q}
    \end{subfigure}
    \hfill
    \begin{subfigure}[t]{0.32\textwidth}
        \centering
        \includegraphics[width=\textwidth]{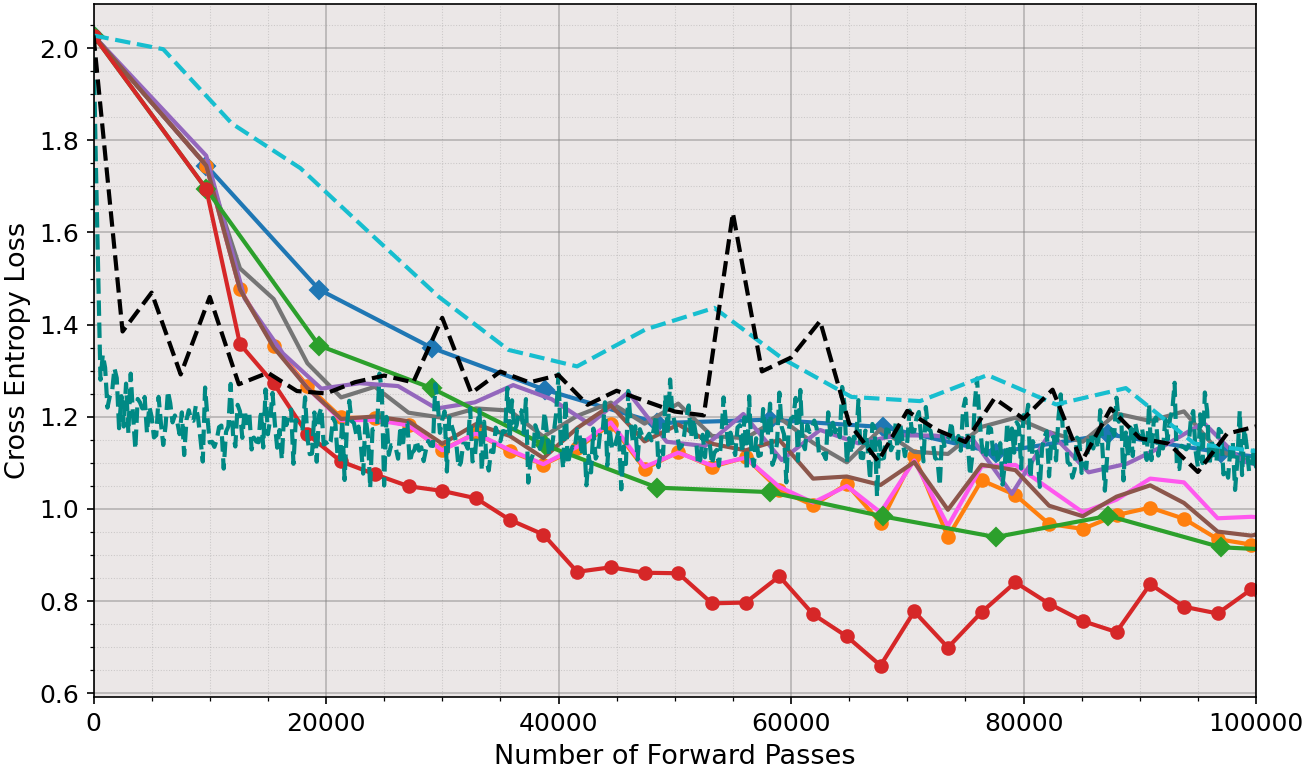}
        \caption{MNIST 8 Qubits 3 Layers}
        \label{fig:training_8q}
    \end{subfigure}
    \hfill
    \begin{subfigure}[t]{0.32\textwidth}
        \centering
        \includegraphics[width=\textwidth]{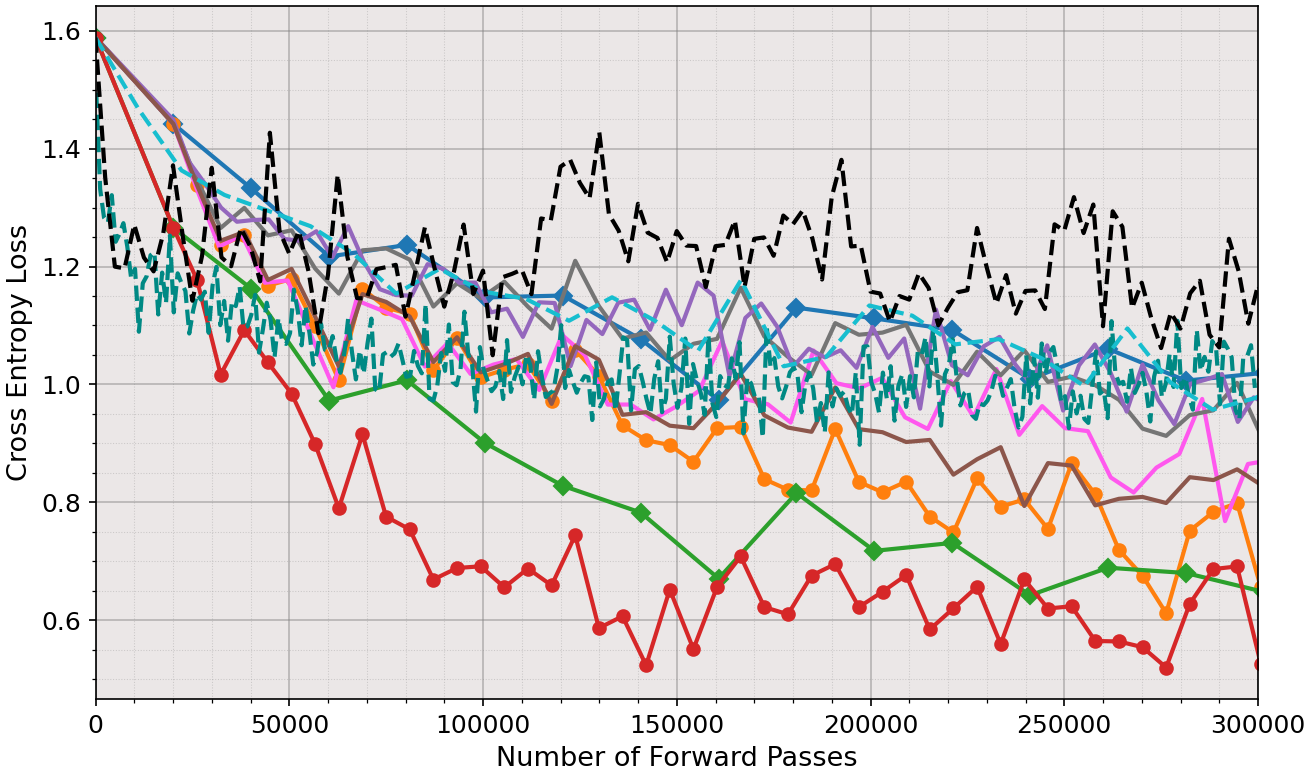}
        \caption{MNIST 10 Qubits 5 Layers}
        \label{fig:training_10q}
    \end{subfigure}
    
    \vspace{0.2cm}
    
    \begin{subfigure}[t]{0.32\textwidth}
        \centering
        \includegraphics[width=\textwidth]{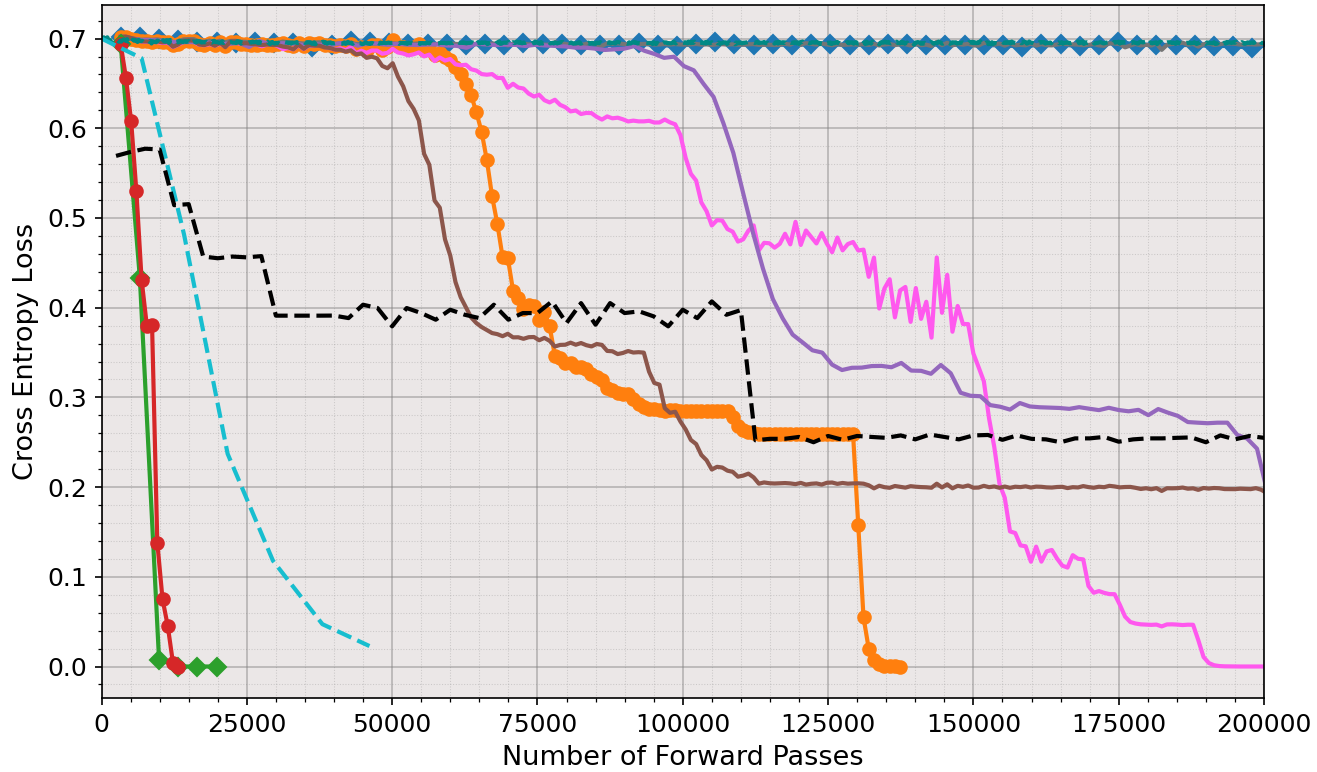}
        \caption{Parity 4 Qubits 2 Layers}
        \label{fig:xor_4q}
    \end{subfigure}
    \hfill
    \begin{subfigure}[t]{0.32\textwidth}
        \centering
        \includegraphics[width=\textwidth]{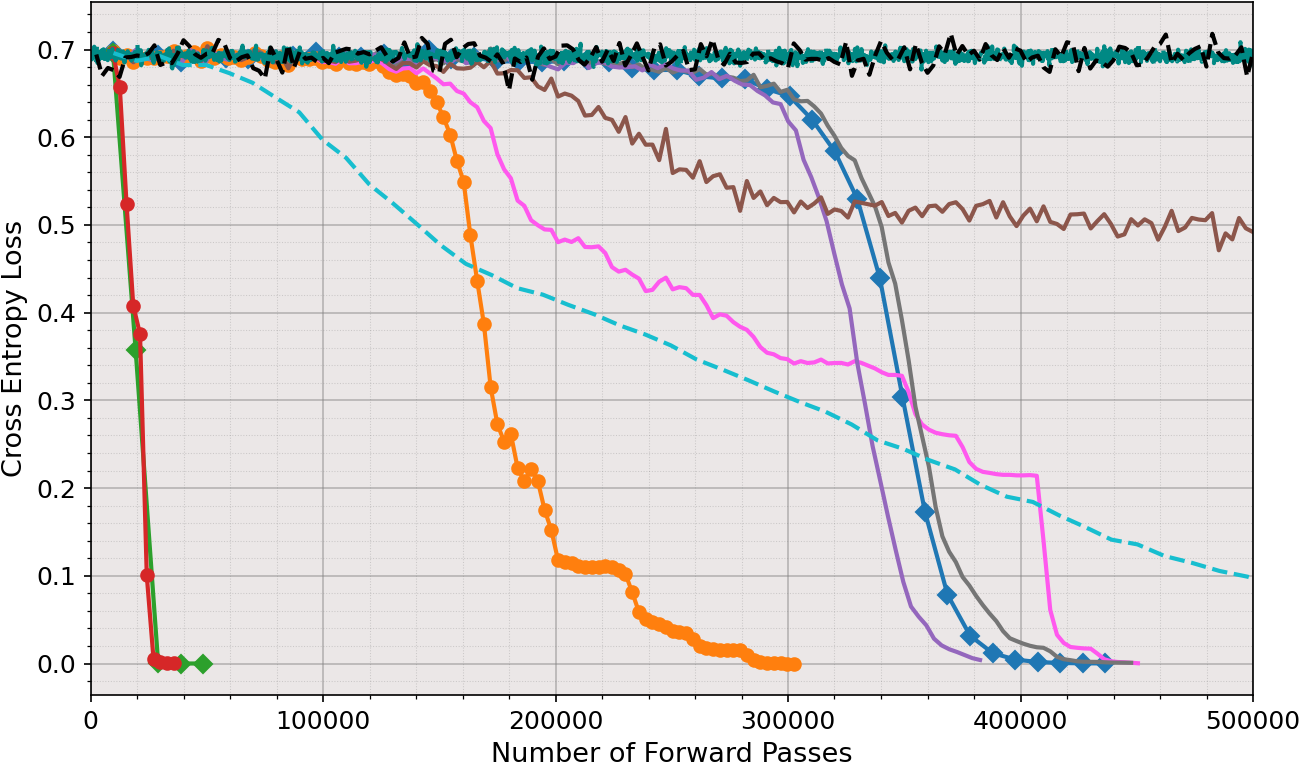}
        \caption{Parity 8 Qubits 3 Layers}
        \label{fig:xor_8q}
    \end{subfigure}
    \hfill
    \begin{subfigure}[t]{0.32\textwidth}
        \centering
        \includegraphics[width=\textwidth]{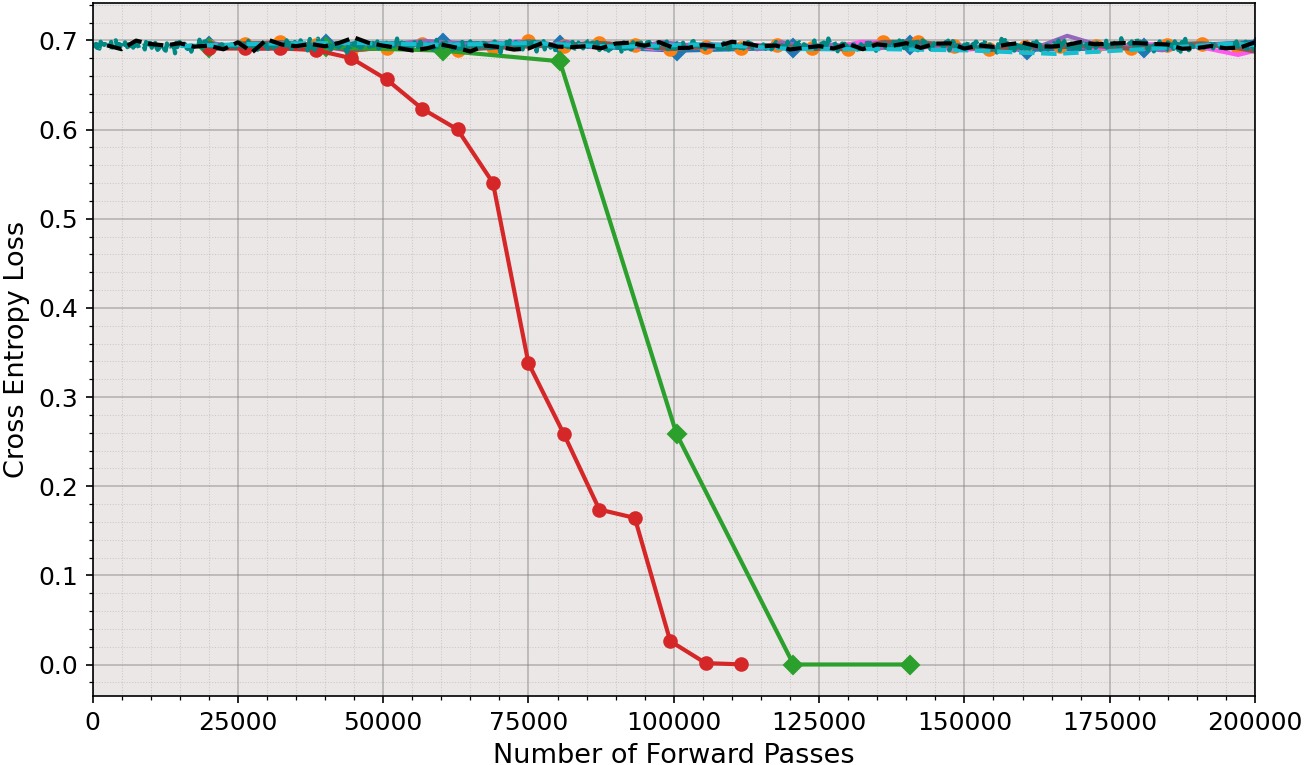}
        \caption{Parity 10 Qubits 5 Layers}
        \label{fig:xor_10q}
    \end{subfigure}

\caption{Comparisons of QNN training on optimizers. Top row: MNIST classification problem. Bottom row: Parity problem. We compare all optimizers summarized in Table \ref{tab:optimizers} on these two tasks. WSBD shows clear performance improvements having both a faster decrease in loss and often reaching a lower loss overall for many models.}
\label{fig:qnn_combined}
\end{figure*}

\begin{figure*}[!h]
    \begin{subfigure}[t]{1.0\textwidth}
        \centering
        \includegraphics[width=\textwidth]{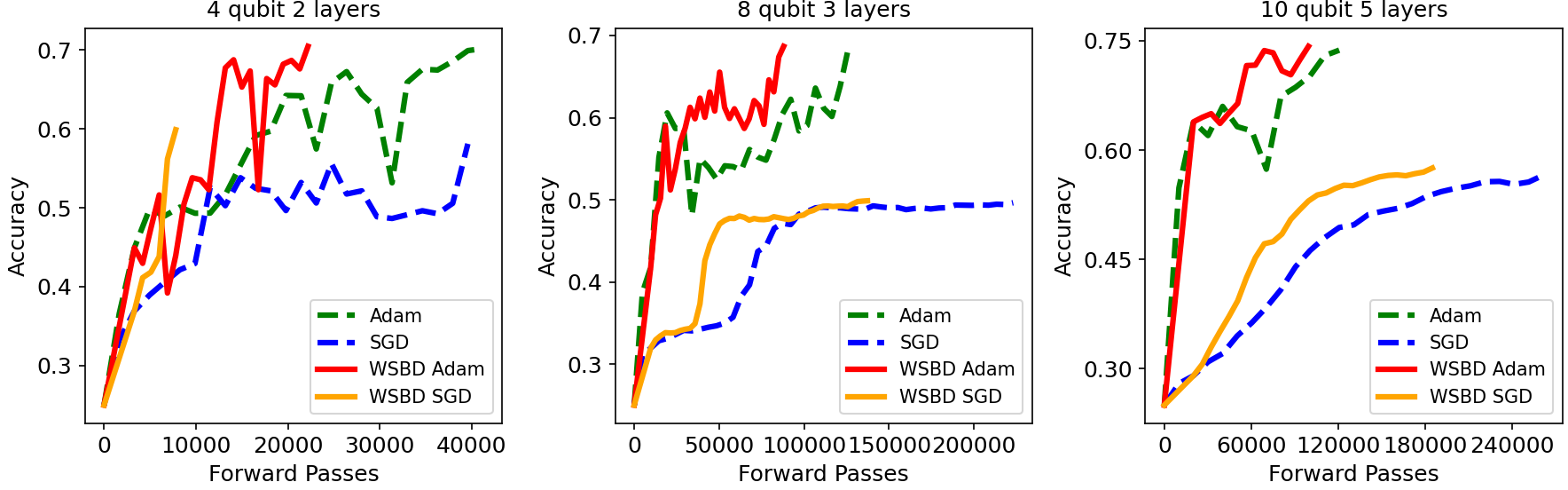}
        \caption{MNIST accuracy training curves with the SGD, Adam and WSBD optimizers. We show the training until the highest accuracy is reached using each optimizer.}
        \label{fig:app_MNIST_ACC_CURVES}
    \end{subfigure}

    \begin{subfigure}[t]{1.0\textwidth}
        \centering
        \includegraphics[width=\textwidth]{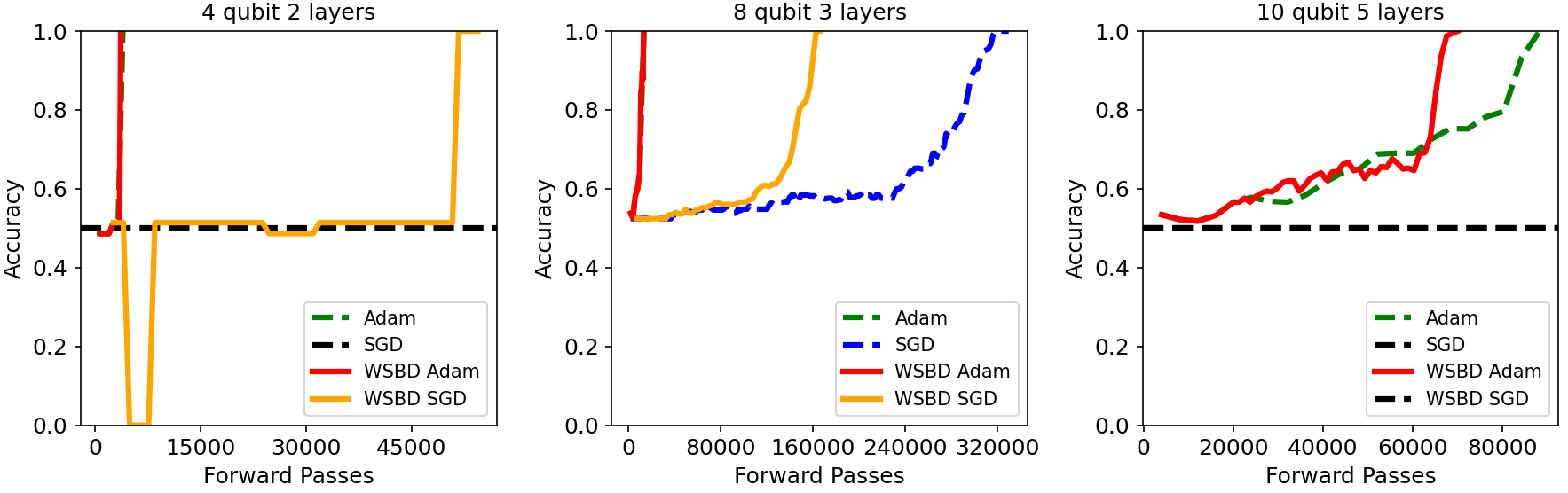}        
        \caption{Parity problem accuracy training curves with the SGD, Adam and WSBD optimizers. The black dotted line represents when a given optimizer wasn't able to solve or improve on the problem.}
        \label{fig:app_PARITY_ACC_CURVES}
    \end{subfigure}
    \caption{Training curves showing how accuracy increases for a 5000 image and 500 bit-string testing set using Adam, SGD and the WSBD optimizers. WSBD reaches its max accuracy faster in all curves and often reaches a higher accuracy overall.}
    \label{fig:combined_accuracy_training_curves}
    
\end{figure*}

\begin{table*}[t]
\centering
\caption{Forward passes to reach max accuracy and target energies for the MNIST, parity and noisy VQE problems. $\mathcal{M}$ represents the number of circuit evaluations each forward pass is ran. For the VQE problem $\mathcal{M} = 1000$. In all problems and QNN sizes, WSBD shows performance improvements reaching its target faster. For the parity problem 0 represents an optimizer not solving or improving on the task.}
\label{tab:vqe_to_taget_energy_DETAILED}
\begin{tabular}{ll|cc|cc}
\toprule
\multirow{2}{*}{\textbf{QNN}} & \multirow{2}{*}{\textbf{Max Accuracy (SGD/Adam}} &
\multicolumn{2}{c|}{\textbf{Standard Optimizers}} &
\multicolumn{2}{c}{\textbf{WSBD Optimizers}} \\
\cmidrule(lr){3-4} \cmidrule(lr){5-6}
\textbf{Size} & \textbf{WSBD-SGD/WSBD-Adam)}  & \textbf{SGD} & \textbf{ADAM} & \textbf{WSBD-SGD} & \textbf{WSBD-ADAM} \\
\midrule
\multicolumn{6}{c}{\textit{MNIST Classification}} \\
\midrule
4q, 2l & (58.1, 70.2 59.9, 70.4) & 39,600$\mathcal{M}$ & 41,250$\mathcal{M}$ & \textbf{7,800}$\mathcal{M}$ & \textbf{22,200}$\mathcal{M}$\\
8q, 3l & (49.6, 68.5 49.9, 68.7) & 223,100$\mathcal{M}$ & 126,100$\mathcal{M}$ & \textbf{137,300}$\mathcal{M}$& \textbf{88,000}$\mathcal{M}$\\
10q, 5l & (56.6, 73.7 57.5, 74.2) & 261,300$\mathcal{M}$ & 120,600$\mathcal{M}$ & \textbf{184,800}$\mathcal{M}$ & \textbf{99,400}$\mathcal{M}$\\
\midrule
\multicolumn{6}{c}{\textit{Parity Problem}} \\
\midrule
4q, 2l & (50.0, 100.0 100.0, 100.0) & 0 & 3,960$\mathcal{M}$ & \textbf{51,720}$\mathcal{M}$ & \textbf{3,660}$\mathcal{M}$\\
8q, 3l & (100.0, 100.0 100.0, 100.0) & 318,160$\mathcal{M}$ & 13,580$\mathcal{M}$ & \textbf{162,820}$\mathcal{M}$ & \textbf{13,180}$\mathcal{M}$\\
10q, 5l & (50.0, 100.0 50.0, 100.0) & 0 & 88,440$\mathcal{M}$ & 0 & \textbf{70,120}$\mathcal{M}$\\
\midrule

\multicolumn{6}{c}{\textit{Ground State Energy Problem}} \\

\midrule
& \textbf{Target Energy (SGD/Adam)}\\
1q, 2l  & $-0.85990 \,/\, -0.88768$ & 1,072$\mathcal{M}$ & 624$\mathcal{M}$  & \textbf{794}$\mathcal{M}$   & \textbf{350}$\mathcal{M}$ \\
1q, 4l  & $-0.90606 \,/\, -0.94134$ & 1,280$\mathcal{M}$ & 5,456$\mathcal{M}$ & \textbf{1,016}$\mathcal{M}$ & \textbf{920}$\mathcal{M}$ \\
1q, 8l  & $-0.71530 \,/\, -0.90632$ & 2,080$\mathcal{M}$ & 9,888$\mathcal{M}$ & \textbf{1,632}$\mathcal{M}$ & \textbf{3,064}$\mathcal{M}$ \\
2q, 1l  & $-1.48394 \,/\, -1.47540$ & 2,416$\mathcal{M}$ & 976$\mathcal{M}$  & \textbf{1,738}$\mathcal{M}$ & \textbf{436}$\mathcal{M}$ \\
2q, 2l  & $-2.00472 \,/\, -2.03302$ & 3,664$\mathcal{M}$ & 4,464$\mathcal{M}$ & \textbf{1,532}$\mathcal{M}$ & \textbf{768}$\mathcal{M}$ \\
2q, 4l  & $-1.81826 \,/\, -1.90246$ & 6,848$\mathcal{M}$ & 4,896$\mathcal{M}$ & \textbf{4,112}$\mathcal{M}$ & \textbf{1,688}$\mathcal{M}$ \\
4q, 1l  & $-3.88842 \,/\, -3.90748$ & 7,216$\mathcal{M}$ & 3,280$\mathcal{M}$ & \textbf{4,760}$\mathcal{M}$ & \textbf{1,316}$\mathcal{M}$ \\
4q, 2l  & $-4.12360 \,/\, -4.09478$ & 11,936$\mathcal{M}$ & 4,384$\mathcal{M}$ & \textbf{9,152}$\mathcal{M}$ & \textbf{2,664}$\mathcal{M}$ \\
4q, 3l  & $- 3.68399 \,/\, -3.75891$ & 13,440$\mathcal{M}$    & 8,496 $\mathcal{M}$    & \textbf{7,200}$\mathcal{M}$ & \textbf{1,994}$\mathcal{M}$ \\
\bottomrule
\end{tabular}
\end{table*}

%%%%%%%%%%%%%%%%%%%%%%%%%%% STD DATA %%%%%%%%%%%%%%%%%%%%%%%%%%%%%%

\begin{table}[h!]
\centering
\begin{tabular}{c|cc|cc}
\textbf{QNN Size} & \textbf{SGD} & \textbf{ADAM} & \textbf{WSBD SGD} & \textbf{WSBD ADAM} \\
\hline
1q 2l & 0.01389 ($\pm$1.39\%) & 0.01163 ($\pm$1.16\%) & 0.01945 ($\pm$1.95\%) & 0.01696 ($\pm$1.70\%) \\
1q 4l & 0.01154 ($\pm$1.15\%) & 0.01012 ($\pm$1.01\%) & 0.01465 ($\pm$1.46\%) & 0.01314 ($\pm$1.31\%) \\
1q 8l & 0.00956 ($\pm$0.96\%) & 0.01325 ($\pm$1.32\%) & 0.02179 ($\pm$2.18\%) & 0.02156 ($\pm$2.16\%) \\

2q 1l & 0.03574 ($\pm$1.60\%) & 0.03630 ($\pm$1.62\%) & 0.03796 ($\pm$1.70\%) & 0.03945 ($\pm$1.76\%) \\
2q 2l & 0.03533 ($\pm$1.58\%) & 0.07152 ($\pm$3.20\%) & 0.05569 ($\pm$2.49\%) & 0.04123 ($\pm$1.84\%) \\
2q 4l & 0.03664 ($\pm$1.64\%) & 0.03352 ($\pm$1.50\%) & 0.04100 ($\pm$1.83\%) & 0.05009 ($\pm$2.24\%) \\
4q 1l & 0.03172 ($\pm$0.67\%) & 0.08518 ($\pm$1.79\%) & 0.09618 ($\pm$2.02\%) & 0.10447 ($\pm$2.20\%) \\
4q 2l & 0.02535 ($\pm$0.53\%) & 0.05149 ($\pm$1.08\%) & 0.05418 ($\pm$1.14\%) & 0.07743 ($\pm$1.63\%) \\
4q 3l & 0.02587 ($\pm$0.54\%) & 0.04386 ($\pm$0.92\%) & 0.05690 ($\pm$1.20\%) & 0.09522 ($\pm$2.00\%) \\
\end{tabular}
\caption{Standard Deviation (STD) and Normalized STD (NSTD = $\frac{\text{STD}}{\text{Ground State Energy}}$) for SGD/ADAM and WSBD variants on the noisy VQE problem. As seen the weighted-stochastic freezing element of WSBD still remains highly stable, with a Normalized STD around 1-2\% typically.}
\end{table}

\begin{table}[h!]
\centering
\begin{tabular}{c|c|c}
\textbf{QNN Size} & \textbf{WSBD SGD (STD, CV)} & \textbf{WSBD ADAM (STD, CV)} \\  \hline
4q 2l  & 0.0228 ($\pm3.25 \cdot 10^{-4}$\%) & 0.0146 ($\pm2.08\cdot10^{-4}$)\% \\ 
8q 3l  & 0.0028 ($\pm7.01 \cdot 10^{-6}$)\% & 0.0132 ($\pm3.35 \cdot 10^{-5}$)\% \\ 
10q 5l & 0.0044 ($\pm8.82 \cdot 10^{-6}$)\% & 0.0076 ($\pm2.04 \cdot 10^{-5}$)\% \\ 

\end{tabular}
\caption{Standard Deviation (STD) and Coefficient of Variation (CV) for WSBD SGD and WSBD ADAM on the MNIST problem. As seen the weighted-stochastic freezing element of WSBD still remains highly stable, with a CV close to 0. Note: Both the parameters and data order was fixed for all runs. Therefore, this experiment tests the stability of the weighted-stochastic freezing aspect directly.}
\end{table}

\end{document}